\documentclass{article}

\usepackage{epsfig, graphics}
\usepackage{latexsym,subfigure}
\usepackage{amsmath,amssymb,amsthm,amsfonts}
\usepackage[numbers]{natbib}
\usepackage{algorithm,algorithmic}

\newtheorem{lemma}{Lemma}
\newtheorem{theorem}{Theorem}

\newcommand\trace[1]{\text{\bf Trace}\left(#1\right)}

\newcommand{\tracer}[2]{\ensuremath{\langle \!\langle {#1}, \; {#2}
\rangle \!\rangle}}

\def\G{\mathcal{G}}
\def\V{\mathcal{V}}
\def\E{\mathcal{E}}
\def\C{\mathcal{C}}

\def\real{\mathbb{R}}

\def\P{\mathcal{P}}
\def\T{\mathcal{T}}
\def\sgn{\text{\bf Sign}}

\title{Clustering using Max-norm Constrained Optimization}

\author{Ali Jalali\\alij@mail.utexas.edu\\University of Texas at Austin\\
\and
Nathan Srebro\\nati@uchicago.edu\\Toyota Technological Institute at Chicago}

\begin{document}
\maketitle





\begin{abstract}
We suggest using the max-norm as a convex surrogate constraint for
clustering.  We show how this yields a better exact cluster recovery
guarantee than previously suggested nuclear-norm
relaxation, and study the effectiveness of our method, and other
related convex relaxations, compared to other clustering approaches.
\end{abstract}

\section{Introduction}
Clustering as the problem of partitioning data into clusters with
strong similarity inside the clusters and strong dissimilarity across
different clusters is one of the main problems in machine learning. In this paper, we consider
the problem of cut-based, or \emph{correlation}, clustering
\cite{BBC02} that has received a lot of attention recently \cite{ailon2010improved,mathieu2010correlation,bagon2011large}: Given $\G(\V,\E)$ on $n$ nodes with
normalized symmetric affinity matrix $A$ (for all $u,v\in\V$: $0\leq
A_{uv}\leq 1$ and $A_{uu}=1$), we want to partition $\V$ into clusters
$\C=\{C_1,\ldots,C_k\}$ so as to minimize the total \emph{disagreement}

\begin{equation}
D(\C) = \sum_{i=1}^k\;\sum_{u,v\in C_i}\;(1-A_{uv}) + \sum_{i\neq j=1}^k\;\sum_{u\in C_i,v\in C_j}\;A_{uv}.
\nonumber
\end{equation}
The first term, captures the \emph{internal} disagreement inside
clusters, and the second term captures the \emph{external} agreement
between nodes in different clusters. In an ideal cluster, the affinities between all members of the
same cluster are $1$ and the affinities between members of two
different clusters are zero and hence the objective is zero. This
objective does not require the number of clusters to be known ahead of
time---we may decide to use any number of clusters, and this is
accounted for in the objective.  Unfortunately, finding a clustering
minimizing the disagreement $D(\C)$ is NP-Hard \cite{BBC02}.

We formulate this problem as an optimization of a convex disagreement
objective over a non-convex set of \emph{valid clustering} matrices
(Section \ref{sec:setup}) and then consider convex relaxations of this
constraint.  Recently, \citet{JCSX11} suggested a trace-norm (aka
nuclear-norm) relaxation, casting the problem as minimizing an
$\ell_1$ loss and a trace-norm penalty, and providing conditions under
which the true underlying clustering is recovered. Instead of
trace-norm, we propose using the max-norm (aka
$\gamma_2:\ell_1\rightarrow\ell_\infty$ norm) \cite{SRERENJAA05}, which is a tighter convex relaxation than the trace-norm.
Accordingly, we establish an exact recovery guarantee for our max-norm
based formulation that is strictly better then the trace-norm based
guarantee.  We show that if the affinity matrix is a
corruption of an ``ideal'' clustering matrix, with a certain bound on
the corruption, then the optimal solution of the max-norm bounded optimization problem is exactly the ideal
clustering (Section \ref{sec:theoretical}).  We also discuss even tighter convex
relaxations related to the max-norm, and suggest augmenting the convex
relaxation with a single-linkage post-processing step in case of
non-exact recovery, showing the empirical advantages of these
approaches (Section \ref{sec:enhancedAlgo}).

The approach we suggests relies on optimizing an $\ell_1$ objective
subject to a max-norm constraint. A similar optimization problem with
a trace-norm constraint (or trace-norm regularization) has recently
been the subject of some interest in the context of ``robust PCA''
\cite{RPCA,RPCAOutlier} and recovering the structure of graphical models with latent
variables \cite{Venkat10}. As with the trace-norm regularized variant,
the $\ell_1$ + max-norm problem can be formulated as an SDP and solved
using standard solvers, but this is only applicable to fairly small
scale problems.  In Section \ref{sec:opt-methods}, we discuss various optimization approaches to this problems, including approaches which preserve the sparsity of the solution.

\subsection{Relationship to the Goemans Willimason SDP Relaxation}

Our convex relaxation approach is related to the classic SDP
relaxations of max-cut \cite{GOEWIL95} and more generally the cut-norm
\cite{NOGASS06}.  In fact, if we are interested in a partition to
exactly two clusters, the correlation clustering problem is
essentially a max-cut problem, though with both positive and negative
weights (i.e.~a symmetric cut-norm problem), and our relaxation is
essentially the classic SDP relaxation of these problems.  Our
approach and results differ in several ways.

First, we deal with problems with multiple clusters, and even when the
number of clusters {\em is not} pre-determined.  If the number of
clusters $k$ {\em is} pre-determined, the correlation clustering
problem can be written as an integer quadratic program, with a $k$
variables per node, and can be relaxed
to an SDP.  But this SDP will be very different from ours, and will
involve a matrix of size $nk \times nk$, unlike our relaxation where
the matrix is of size $n\times n$ regardless of the number of
clusters.  Consequently, the rounding techniques based on (random)
projections typically employed for classic SDP relaxations do not seem
relevant here.  Instead, we employ a single-linkage post-processing as
a form of ``rounding'' imperfect solutions.

Second, the type of guarantees we provide are very different from
those in the Theory of Computation literature.  Most of the SDP
relaxation work we are aware of (including the classical work cited
above) focuses on worst case constant factor approximation guarantees.
On one hand, this means the guarantee needs to hold even on ``crazy''
inputs where there is really no reasonable clustering anyway, and
second, and on the other hand it is not clear how approximating the
objective to within a constant factor translates to recovering an
underlying clustering.  Instead, we prove that when the affinity
matrix is close enough to following some underlying ``true''
clustering, the true clustering will be recovered {\em exactly}.  This
type of guarantee is more in the spirit of compressed sensing, which
where {\em exact} recovery of a support set is guaranteed subject to
conditions on the input \cite{JCSX11}.

\subsection{Other Clustering Approaches}
There are several classes of clustering algorithms with different objectives. In hierarchical
clustering algorithms such as UPGMA \cite{SNESOK73}, SLINK
\cite{SIB73} and CLINK \cite{DEF77} the goal is to generate a sequence of clusterings by  produce a sequence of clustering by merging/splitting two clusters at each step of the sequence according to a \emph{local} disagreement objective as opposed to our global $D(\C)$. Because of this locality, these methods are known to be very sensetive to outliers.
 
Cut-based clustering algorithms such as $k$-means/medians
\cite{STE57,JAIDUB81}, ratio association \cite{SHIMAL}, ratio cut
\cite{CHASCHZIE} and normalized cut \cite{YUSHI} try to optimize an
objective function globally. The main issue with these objectives is that they are typically NP-Hard and need to know the number of clusters ahead of time, since these objectives are monotone in the number of clusters.

In contrast, spectral clustering algorithms\cite{LUX07} try to find the first $k$ principal component of the affinity matrix or a transformed version
of that \cite{MELSHI01}. These methods require the number of
clusters in advance and has been shown to be
tractable (convex) relaxations to NP-Hard cut-based algorithms
\cite{DHIGUAKUL}. These methods are
again very sensitive to outliers as they might change the principal
components dramatically.

\section{Problem Setup} \label{sec:setup}
Our approach is based on representing a clustering $\C$ through its
incidence matrix $K(\C)\in \mathbb{R}^{n\times n}$ where $K_{uv}=1$
iff $u$ and $v$ belong to the same cluster in $\C$ (i.e.~$u,v\in C_i$
for some $i$), and $K_{uv}=0$ otherwise (i.e.~if $u$ and $v$ belong to
different clusters).  The matrix $K(\C)$ is thus a permuted
block-diagonal matrix, and can also be thought of as the edge incidence
matrix of a graph with cliques corresponding to clusters in $\C$.  We
will say that a matrix $K$ is a {\bf valid clustering matrix}, or
sometimes simply {\bf valid}, if it can be written as $K=K(\C)$ for
some clustering $\C$ (i.e.~if it is a permuted block diagonal matrix,
with $1$s in the diagonal blocks).

The disagreement can then be written as either:
\begin{equation}\label{eq:absobj}
D(\C) = \|A - K(\C)\|_1 = \sum_{u,v}\;\left|A_{uv}-K(\C)_{uv}\right|
\end{equation}
or as:
\begin{equation}
  \label{eq:linearobj}
D(\C) = \sum_{u,v} K(\C)_{uv} (1-2 A_{uv}) + \sum_{uv} A_{uv}~,
\end{equation}
where the term $\sum_{uv} A_{uv}$ does not depend on the clustering
$\C$ and can thus be dropped.

We now phrase the correlation clustering problem as matrix problem,
where we would like to solve
\begin{equation}
\begin{aligned}
\min_{K}\,D(K)
\;\;\text{s.t.}\quad K\text{ is a valid clustering matrix.}\\
\end{aligned}
\label{eq:orig-optimization}
\end{equation}
 The problem is that
even though the objectives \eqref{eq:absobj} and \eqref{eq:linearobj}
are convex, the constraint that $K$ is valid is certainly not
constraint.  Our approach to correlation clustering will thus be to
relax this non-convex constraint (the validity of $K$) to a convex
constraint.

We note that although both the absolute error objective
\eqref{eq:absobj} and the linear objective \eqref{eq:linearobj} agree
on valid clustering matrices (or more generally, on binary matrices
$K$), they can differ when $K$ is fractional, and especially when $A$
is also fractional.  The choice of objective can thus be important when
relaxing the validity constraint to a convex constraint. More specifically, as long as $A$ is binary (i.e.  $A_{uv} \in
\{0,1\}$), and $0\leq K_{uv} \leq 1$, even if $K$ is fractional, the
two objectives agree.  Non-negativity of $K_{uv}$ is ensured in some,
but not all, of the convex relaxations we study.  When non-negativity
is not ensured, the absolute error objective \eqref{eq:absobj} would
tend to avoid negative values, but the linear objective might
certainly prefer them. More importantly, once the affinities $A_{uv}$ are also fractional,
the two objectives differ even for $0\leq K_{uv} \leq 1$.  While the
linear objective would tend to not care much about entries with
affinities close to $1/2$, the absolute error objective would tend to
encourage fractional values in thees cases.

The linear objective also has some optimization advantages over the
absolute function as well. From a numerical optimization point of
view, dealing with the linear objective function is easier since we do
not need to compute the sub-gradients of the $\ell_1$-norm.

\section{Max-Norm Relaxation} \label{sec:basicAlgo}
As discussed in the previous Section, we are interested in optimizing
over the non-convex set of valid clustering matrices.  The approach we
discuss here is to relaxing this set to the set of matrices with
bounded {\em max-norm} \cite{SRERENJAA05}.  The max-norm of a matrix $K$ is
defined as
\begin{equation}
\|K\|_{\max} = \min_{K=RL^T} \;\|R\|_{\infty,2}\|L\|_{\infty,2}
\nonumber
\end{equation}
where, $\|\cdot\|_{\infty,2}$ is the maximum of the $\ell_2$ norm of
the rows, and the minimization is over factorization of any internal
dimensionality.  It is not hard to see that if $K$ is a valid
clustering matrix, with $K=K(\C)$, then $\|K\|_{\max}=1$.  This is
achieved, e.g., by a factorization with $R=L$, and where each row
$R_u$ of $R$ is a (unit norm) indicator vector with $R_{ui}=1$ for
$u\in C_i$ and zero elsewhere.

Relaxing the validity constraint to a max-norm constraint, and using
the absolute error objective, we obtain the following convex
relaxation of the correlation clustering problem:
\begin{equation}
\begin{aligned}
\widehat{K}\; =\;\arg\;&\min_{K}\;\|A-K\|_1\;\;\;\text{s.t.}\quad\|K\|_{\max}\leq 1.
\end{aligned}
\label{eq:relax-optimization}
\end{equation}
Alternatively, we could have used the linear objective
\eqref{eq:linearobj} instead.  In any case, after finding
$\widehat{K}$, it is easy to check whether it is valid, and if so
recover the clustering from its block structure.  If $\hat{K}$ is
valid, we are assured the corresponding clustering is a globally
optimal solution of the correlation clustering problem.

\subsection{Theoretical Guarantee} \label{sec:theoretical}
Assuming there exists an underlying true clustering, we provide a worst-case (deterministic) guarantee for exact recovery of that clustering in the presence of noise when the affinity matrix $A$ is a binary $0-1$ matrix using absolute objective. The flavor of our result is similar to \cite{JCSX11} for trace-norm, except that we show the max-norm constraint problem recovers the underlying clustering with larger noise comparing to trace-norm constraint. This matches our intuition that max-norm is a tighter relaxation than trace-norm for valid clustering matrices.

To present our theoretical result, we start by introducing an
important quantity that our main result is based upon. Suppose
$\mathcal{C}^*=\{C_1^*,\ldots,C_k^*\}$ is the underlying true
clustering. For a node $u$ and a cluster $C_i^*$, let $d_{u,C_i^*} =
\frac{\sum_{v\in C_i^*}A_{u,v}}{|C_i^*|}$ if $u\notin C_i^*$ and
$d_{u,C_i^*}=1-\frac{\sum_{v\in C_i^*}A_{u,v}}{|C_i^*|}$ otherwise and
\begin{equation}D_{\max}(A,K) \equiv D_{\max}(A,K(\mathcal{C}^*)) =
\max_{u,i}\;d_{u,C_i^*}\nonumber\end{equation}
be the maximum of the disagreement ratios on
the adjacency matrix. This definition is inspired by \cite{JCSX11} but
is slightly different. Notice that the larger $D_{\max}(A,K)$ is, the
more noisy (comparing to ideal clusters) the graph is; and hence, the
harder the clustering becomes. In particular for ideal clusters (fully
connected inside and fully disconnected outside clusters), we have
$D_{\max}(A,K)=0$. 

We would like to ensure that when $D_{\max}(A,K)$ is small enough, our
method can recover $K$.  The following lemma helps us understand the
information theoretic limit of $D_{\max}(A,K)$, i.e.~what value of
$D_{max}$ is certainly {\em not} enough to ensure recovery, even
information theoretically:

\begin{lemma}
For any clustering $\mathcal{C}=\{C_1,\ldots,C_k\}$ and for all $\gamma>\frac{2}{5+r}$ with $r=\frac{n^2}{\sum_i\, |C_i|^2}$, there exists an affinity matrix $A$ such that $D_{\max}(A,K(\mathcal{C}))=\gamma$ and the combinatorial program \eqref{eq:orig-optimization} does not output $\mathcal{C}$.
\label{lem:DmaxBound}
\end{lemma}

\begin{figure}[t]
\centering
\includegraphics[width=3.5in]{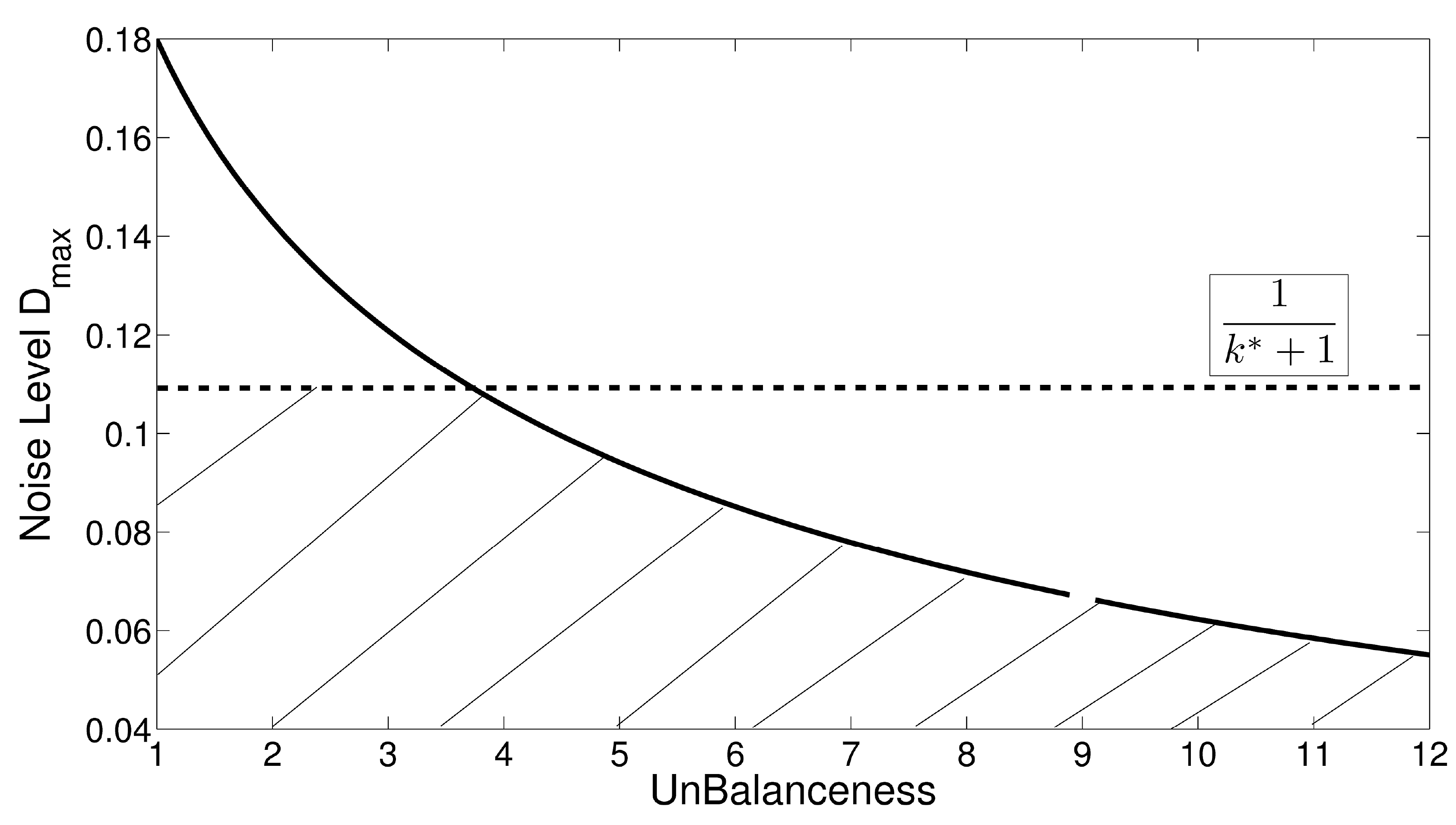}
\caption{Theorem~\ref{thm:deterministic} guarantee region of the noise level $D_{\max}$ vs the unbalanceness parameter $\frac{1}{k^*}\sum_i\left(\frac{|C^*_i|}{|C_{\min}|}\right)^2$.}
\label{fig:region}
\end{figure}

Note that the minimum of $\frac{2}{5+r}$ is attained when all clusters have equal sizes. If we have $k^*$ clusters of size $\frac{n}{k^*}$, then $r=k^*$ and the bound in Lemma~\ref{lem:DmaxBound} asserts that if $D_{\max}(A,K)>\frac{2}{k^*+5}$, then there are examples for which the original clustering cannot be recovered by the combinatorial program \eqref{eq:orig-optimization}. This implies that $D_{\max}(A,K)$ cannot be scaled better than $\Theta(\frac{1}{k^*})$ in general even without convex relaxation.

Suppose there exist a true underlying clustering $\mathcal{C}^*$ with $k^*$ clusters. Let $C_{min}$ be the smallest size underlying true cluster and we are given an affinity matrix $A$ with $D_{\max}=D_{\max}(A,K(\mathcal{C}^*))$. Introducing lagrange multiplier $\mu$, we consider the optimization problem
\begin{equation}
\begin{aligned}
\widehat{K}_{\mu}\; =\;\arg\;&\min_{K}\;\frac{1-\mu}{n^2}\,\|A-K\|_1\;+\;\mu\,\|K\|_{\max}.
\end{aligned}
\label{eq:equiv-relaxed-optimization}
\end{equation}
 The following theorem characterizes the noise regime under which the simple max-norm relaxation \eqref{eq:equiv-relaxed-optimization} recovers $\mathcal{C}^*$.

\begin{theorem}
For binary $0-1$ matrix $A$, if $D_{\max}<\frac{1}{k^*+1}$ is small enough to satisfy $\frac{1}{k^*}\sum_i\left(\frac{|C^*_i|}{|C_{\min}|}\right)^2\leq\frac{(1-3D_{\max})^2}{(1+D_{\max})D_{\max}}$ then, for any $\mu_0$ satisfying $\frac{(1+D_{\max})}{(1-3D_{\max})|C_{\min}|^2} <\frac{(1-\mu_0)k^*}{\mu_0 n^2} <\frac{(1-3D_{\max})k^*}{D_{\max}\sum_i|C_i^*|^2}$, the matrix $\widehat{K}_{\mu_0}$ (the solution to \eqref{eq:equiv-relaxed-optimization}) is unique and equal to the matrix $K^*=K(\mathcal{C}^*)$ (the solution to \eqref{eq:orig-optimization}).
\label{thm:deterministic}
\end{theorem}

\noindent {\bf Remark 1:} Consider the parameter $\frac{1}{k^*}\sum_i\left(\frac{|C^*_i|}{|C_{\min}|}\right)^2$ in the theorem. Notice that for a balanced underlying clustering ($k^*$ clusters of size $n/k^*$), this parameter is $1$ and as the underlying clustering gets more and more unbalanced, this parameter increases. That motivates to call it \emph{unbalanceness} of the clustering. It is clear that as unbalanceness parameter increases, the region of $D_{\max}$ for which our theorem guarantees the clustering recovery shrinks. We plot the admissible region of $D_{\max}$ due to unabalanceness in Fig~\ref{fig:region}.

\noindent {\bf Remark 2:} According to the Lemma~\ref{lem:DmaxBound}, the bound on $D_{\max}$ is order-wise tight and can be only improved by a constant in general.

\begin{figure}[t]
\centering
\includegraphics[width=2.2in]{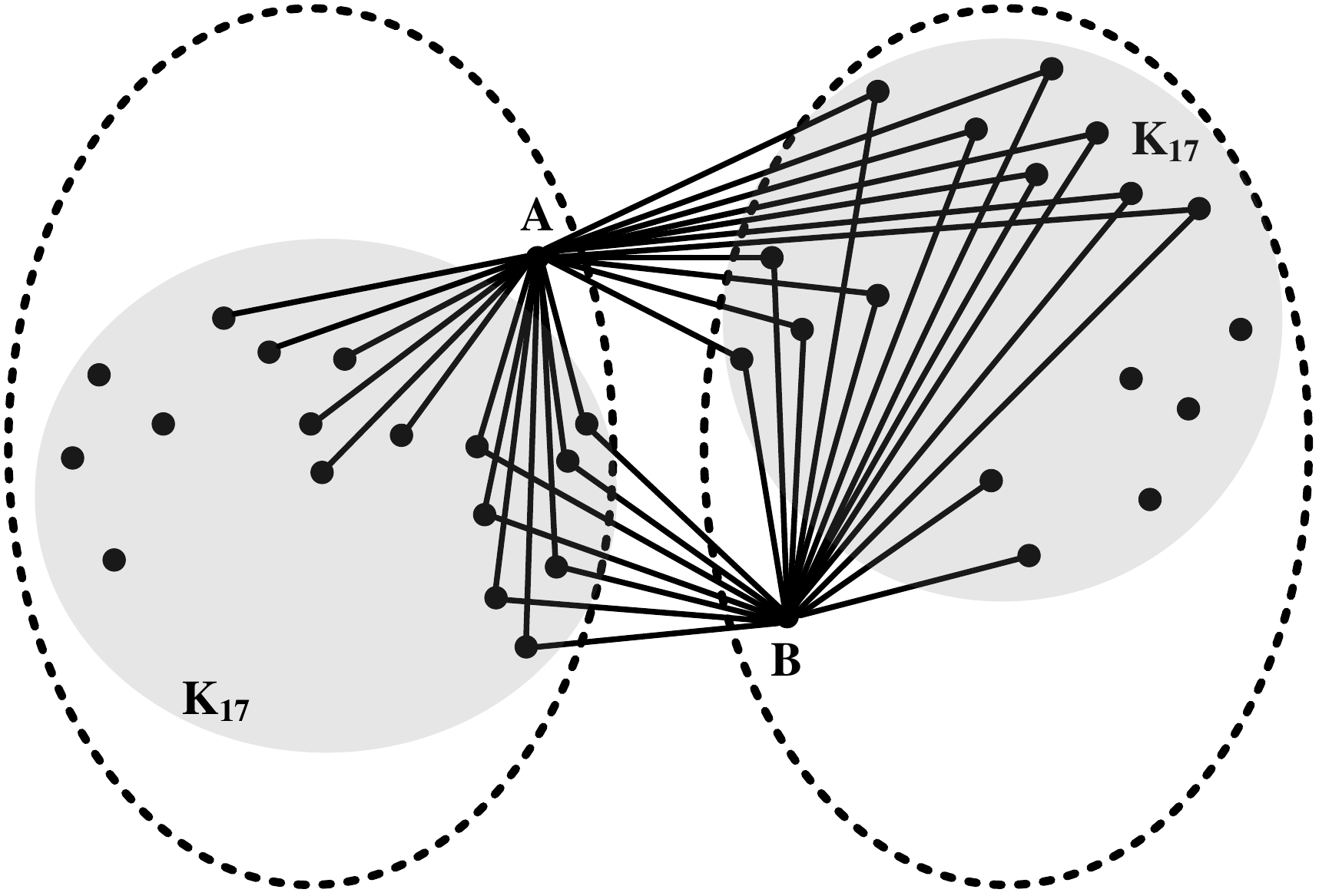}
\caption{Starting from two ideal clusters of size 18, we disconnect node $A$ from 4 nodes on the left cluster and connect it to 11 nodes on the right cluster. Moreover, we disconnect node $B$ from 4 nodes on the right cluster and connect it to 7 nodes on the left cluster as shown. The optimal clustering according to \ref{eq:orig-optimization} is still the original two clusters.}
\label{fig:countexamp}
\end{figure}

\begin{figure*}[t]
\centering
\subfigure[Balanced; Binary]{
\includegraphics[width=0.4\linewidth]{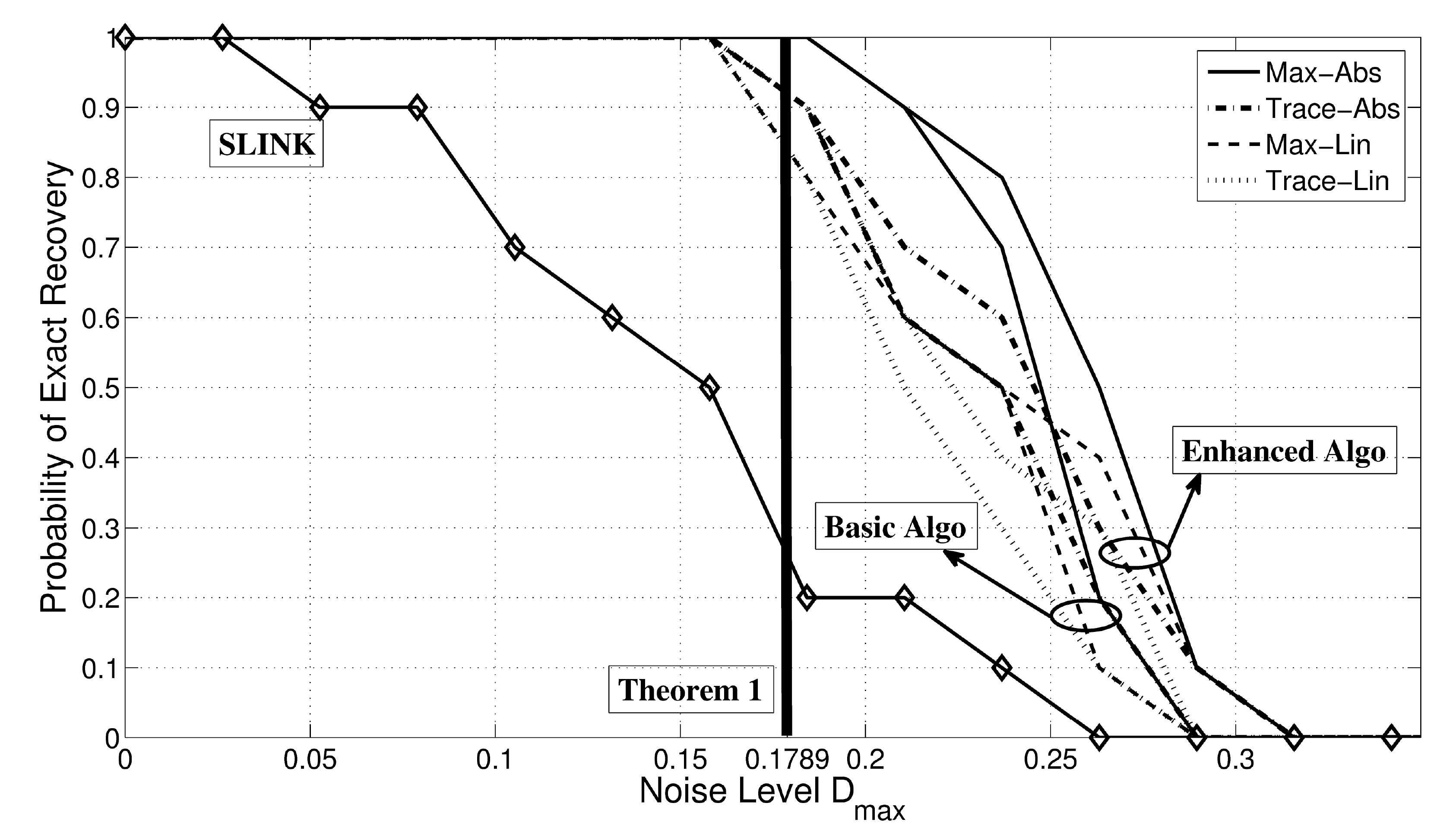}
\label{fig:subfig-bal-int-exact}
}
$\qquad\qquad$
\subfigure[UnBalanced; Binary]{
\includegraphics[width=0.4\linewidth]{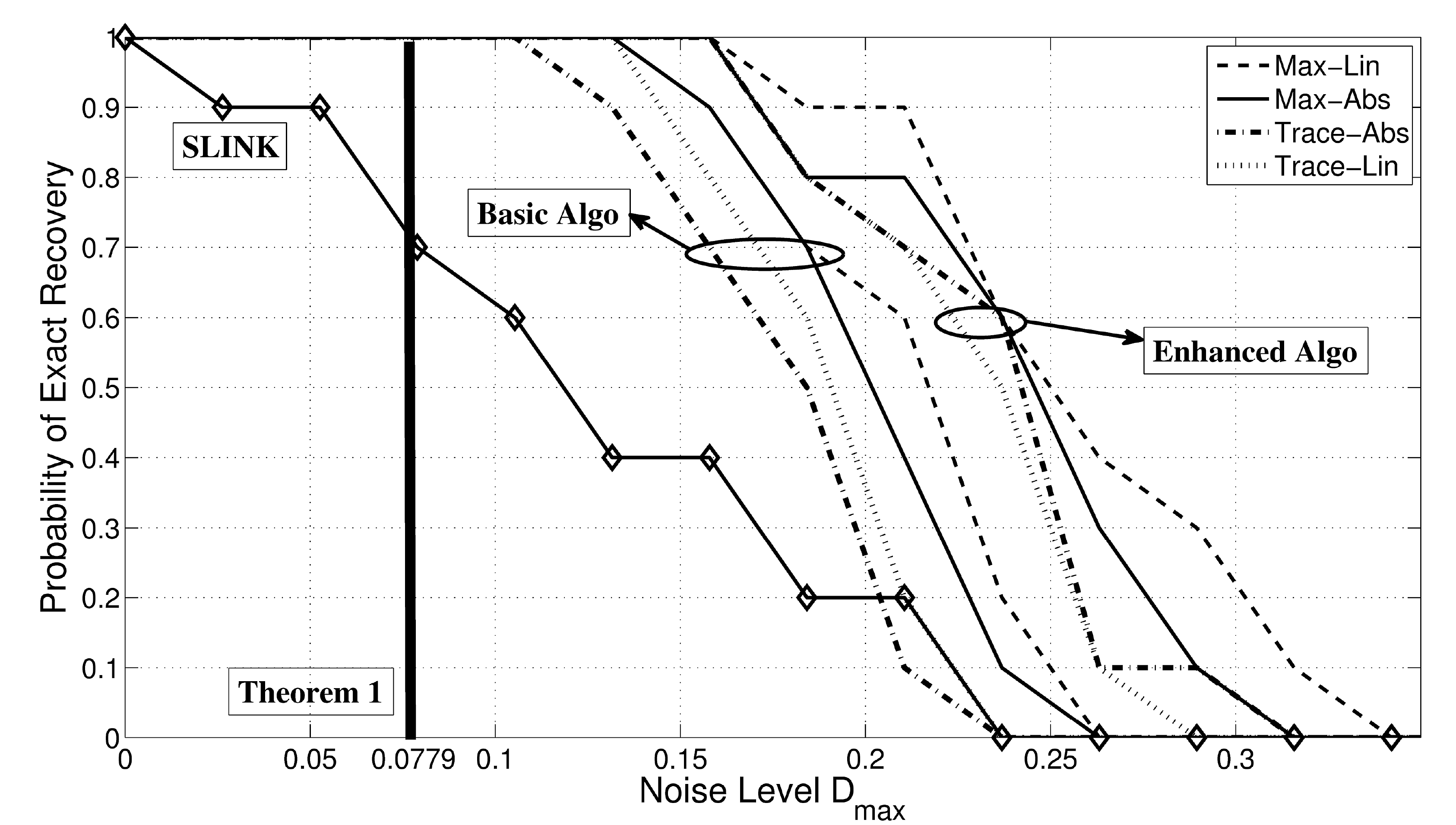}
\label{fig:subfig-unbal-int-exact}
}
\subfigure[Balanced; Fractional]{
\includegraphics[width=0.4\linewidth]{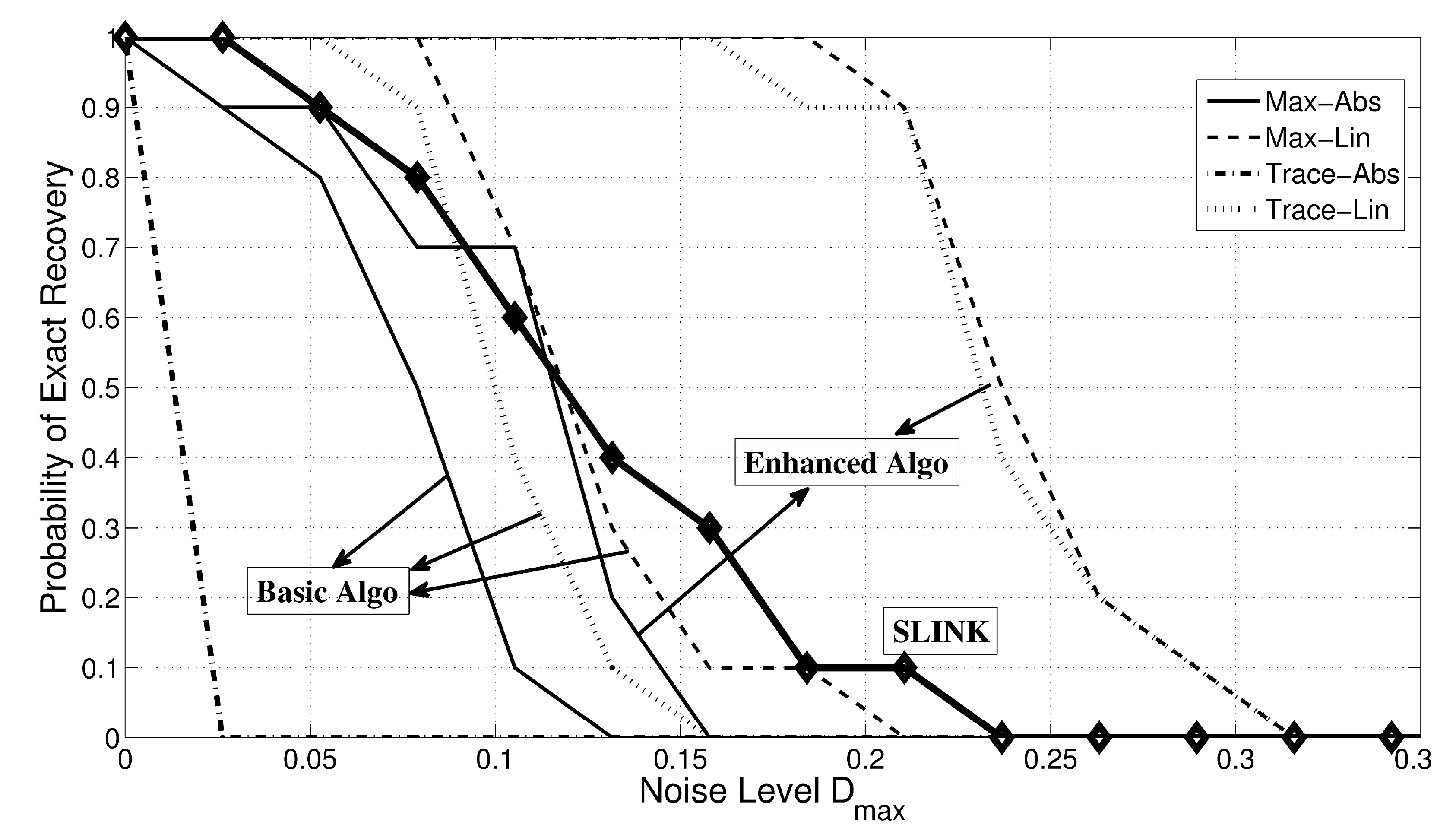}
\label{fig:subfig-bal-frac-exact}
}
$\qquad\qquad$
\subfigure[UnBalanced; Fractional]{
\includegraphics[width=0.4\linewidth]{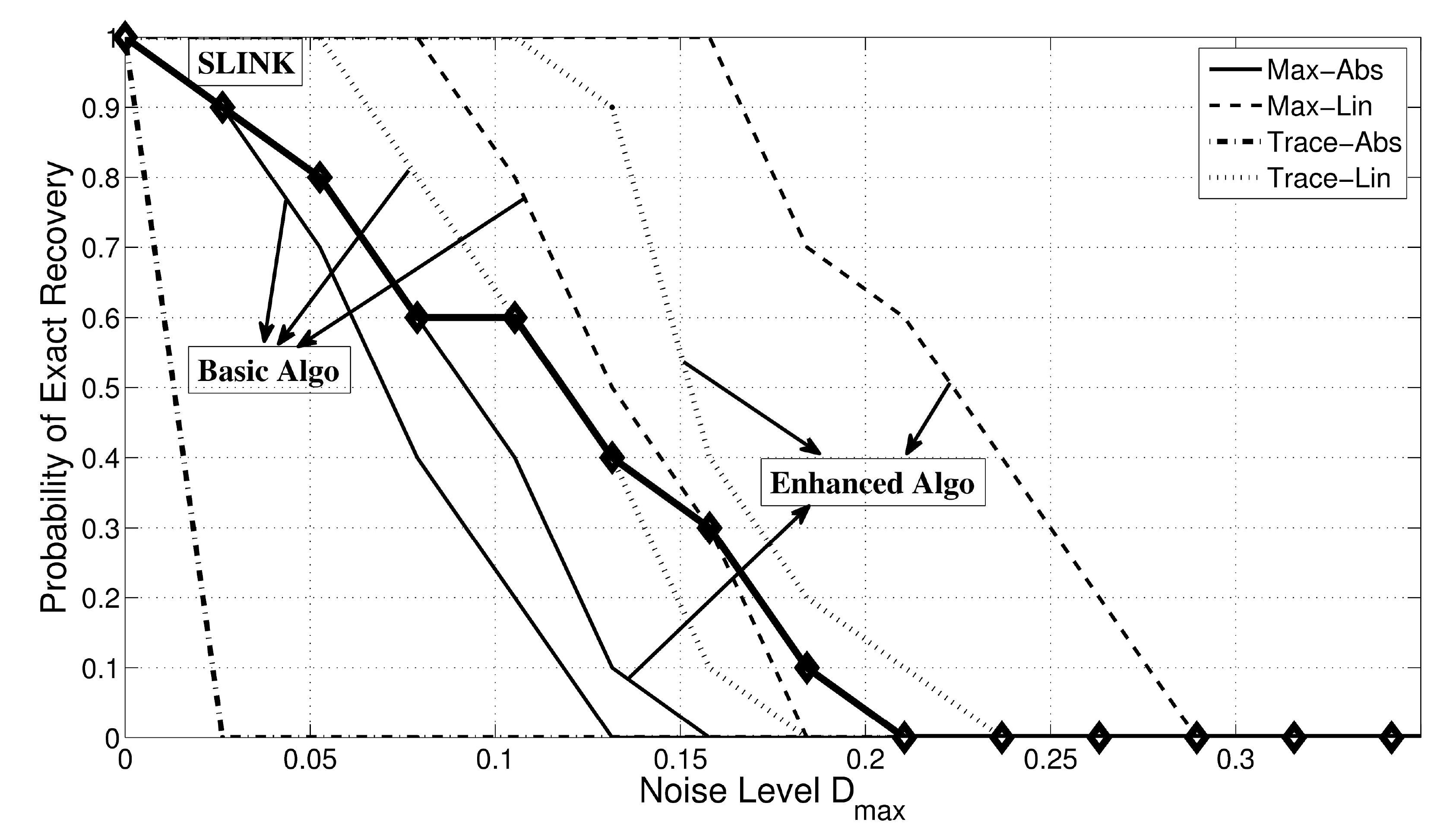}
\label{fig:subfig-unbal-frac-exact}
}
\label{fig:fig2}
\caption{Probability of exact clustering recovery for max-norm and trace-norm constrained algorithms under absolute $\|A-K\|_1$ and linear $\sum_{i,j} K_{ij}(1-2A_{ij})$ objectives. There are $4$ clusters of size $25$ for the balanced case and three clusters of size $30$ + one cluster of size $10$ for the unbalanced case. We consider two cases for each graph; where the affinity matrix is binary and when it is not. We both show the results for simple max-norm relaxation (basic algorithm) and tighter relaxations presented in Section~\ref{sec:enhancedAlgo} (enhanced algorithm). The result shows that max-norm constrained optimization recovers the exact clustering matrix under higher noise regimes better than trace-norm and single-linkage algorithm. Also, the linear objective seems to be performing better than the absolute objective for the clustering problem in most cases.}
\end{figure*}

\subsection{Comparison to Single-Linkage Algorithm}

Considering single-linkage algorithm (SLINK) \cite{SIB73} as a baseline for
clustering, we compare the power our algorithm in cluster recovery
with that. SLINK generates a hierarchy of clusterings starting with
each node as a cluster. At each iteration, SLINK measures the
similarity of all pairs of clusters and combines the most similar pair
of clusters to a new cluster. We consider the closedness of the
columns $A_i$ and $A_j$ as the similarity measure of nodes $i$ and
$j$. 

Consider the graph shown in Fig.~\ref{fig:countexamp}. With exhaustive search, one can show that the non-convex problem \eqref{eq:orig-optimization} outputs two clusters as shown.  Running SLINK on this graph, the algorithm first finds two cliques of size 17 and nodes $A$ and $B$ as four separate clusters in the hierarchy. Next, it combines nodes $A$ and $B$ as a separate cluster since they are more similar to each other than to their own clusters. This means that single linkage algorithm will never find the correct clustering. However, it can be easily checked that our proposed max-norm constrained algorithm will recover the solution of \eqref{eq:orig-optimization}.

\subsection{Comparison to Trace-Norm Constrained Clustering} \label{sec:trace-comparison}

Since the max-norm constraint is strictly a tighter relaxation to the trace-norm constraint, we expect the max-norm algorithm to perform better. Our theorem shows improvement over the guarantees provided for trace-norm clustering. Comparing to the result of \cite{JCSX11} on trace-norm ($D_{\max}\leq\frac{|C_{\min}|}{4n}$), the max-norm tolerates more noise. To see this, consider a balanced clustering, then trace-norm requires $D_{\max}\leq\frac{1}{4k^*}$ and max-norm requires $D_{\max}\leq\min(\frac{1}{k^*+1},0.1789)$ which is larger than $\frac{1}{4k^*}$ for all $k^*$. The difference gets more clear for unbalanced clustering. Suppose we have one small cluster of constant size $|C_{\min}|$ and other clusters are approximately of size $\frac{n}{k^*}$. As $(n,k^*)$ scales, trace-norm guarantee requires that $D_{\max}=o(\frac{1}{n})$ which is inverse proportional to the size of the smallest cluster, whereas, max-norm guarantee requires $D_{\max}=o(\frac{k^*}{n})$ which is inverse proportional to the size of the largest cluster. This is a huge theoretical advantage in our theorem.

Besides comparing the provided guarantees, we compare max-norm clustering with trace-norm clustering both deterministically and probabilistically. Running Trace-Norm constrained minimization \cite{JCSX11} on the graph shown in Fig.~\ref{fig:countexamp}, the resulting clustering consists of two clusters and node $B$ belongs to the correct cluster. However, node $A$ belongs to both clusters! -- The clustering matrix contains two blocks of ones and the row/column corresponding to the node $A$ contain all ones. Also, the diagonal entry corresponding to node $A$ is larger than one and the diagonal entry corresponding to the node $B$ is less than one. In short, this algorithm is confused as of which cluster the node $A$ belongs to.

Further, we compare our algorithm with trace-norm algorithm \cite{JCSX11} and SLINK on a probabilistic setup. Start from two different ideal clusters on $100$ nodes: a) \emph{Balanced} clusters: four ideal clusters of size $25$, b) \emph{Unbalanced} clusters: three ideal clusters of size $30$ and one ideal cluster of size $10$. Then, gradually increase $D_{\max}$ on both graphs and run all algorithms and report the probability of success in exact recovery of the underlying clusters. Although our theoretical guarantee is for binary affinity matrices, here, we run the same experiment for fractional affinity matrix. We run all experiments for both absolute and linear objectives. Fig.~\ref{fig:fig2} shows that in all cases max-norm outperforms the trace-norm and the improvement is more significant for unbalanced clustering with fractional affinity matrix. Moreover, this experiments reveal that the absolute objective has slight advantage if the affinity matrix is binary and clusters are balanced; otherwise, the linear objective is better.

\section{Max-norm + $\ell_1$-norm Optimization} \label{sec:opt-methods}
In this Section we consider optimization problems of the form
\eqref{eq:relax-optimization}. This problem recovers a sparse and low-rank matrix from their sum, considering max-norm as a proxy to rank. In Section \ref{sec:SDP}, we discuss how \eqref{eq:relax-optimization} can be formulated as an SDP, allowing us to easily solve it using standard SDP solvers, as long as the problem size is relatively small. We then propose three other methods to numerically solve the optimization problem \eqref{eq:relax-optimization}.

\subsection{Semi-Definite Programming Method} \label{sec:SDP}

Following \citet{SRERENJAA05}, we introduce dummy variables
$L,R\in\real^{n\times n}$ and reformulate
\eqref{eq:relax-optimization} as the following SDP problem
\begin{equation}
\begin{aligned}
  \widehat{K} = \arg\,&\min_{K,L,R}\; \|A - K\|_1\\
  &\;\text{s.t.}\quad \left[\begin{tabular} {cc}$L\,$&$\,K$\\
      $K^T$&$R$ \end{tabular}\right]\succeq 0\quad\text{and}\quad
  L_{ii},R_{ii}\leq 1
\end{aligned}
\nonumber
\end{equation}
These constraints are equivalent to the condition $\|K\|_{\max}\leq
1$. This SDP can be solved using generic SDP solvers, though is very slow and is not scalable to large problems.

\subsection{Factorization Method}

Motivated by \citet{LRSST10}, we introduce dummy variables $L,R\in\real^{n\times n}$ and let $K=LR^T$. With this change of variable, we can reformulate \eqref{eq:relax-optimization} as

\begin{equation}
\begin{aligned}
\widehat{K} = \widehat{L}\widehat{R}^T = \arg\,&\min_{L,R}\; \|A - LR^T\|_1\\
&\;\text{s.t.}\quad \|L\|_{\infty,2},\|R\|_{\infty,2}\leq 1. 
\end{aligned}
\nonumber
\end{equation}
This problem is not convex, but it is guaranteed to have no local minima for large enough size of the problem \cite{BURCHO06}. Furthermore, if we now the optimal solution $\hat{K}$ has rank at most $r$, we can take $L,R$ to be $\real^{n \times (r+1)}$. In practice, we truncate to some reasonably high rank $r$ even without a known gurantee on the rank of the optimal solution. To solve this problem iteratively, \citet{LRSST10} suggest the following update
\begin{equation}
\left[\begin{tabular}{c}$L$\\$R$\end{tabular}\right]_{k+1}\!\!\!\!\!\!\!\! = \P_{\max}\left(\left[\begin{tabular}{c}$L$\\$R$\end{tabular}\right]_{k}\!\!\!\!\! +\frac{\tau}{\sqrt{k}}\left[\begin{tabular}{c}$\sgn(A-LR^T)\;\;R$\\$\sgn(A-LR^T)^T\,L$\end{tabular}\right]_{k}\right).
\nonumber
\end{equation}
The projection $\P_{\max}(\cdot)$ operates on rows of $L$ and $R$; if
$\ell_2$-norm of a row is less than one, it remains unchanged,
otherwise it will be rescaled so that the $\ell_2$-norm becomes
one.

A possible problem with the above formulation is the lack of
``sparsity" in the following sense:  The $\ell_1$ objective is likely
to yield and optimal solution $K^*$ with many non-zeros in $A-K^*$,
i.e.~where $K^*$ is {\em exactly} equal to $A$ on some of the
entries.  However, gradient steps on the factorization are not likely
to end up in exactly sparse solutions, and we are not likely to see
any such sparsity in solutions obtained by the above method.  

\subsection{Loss Function Method}
There are gradient methods such as truncated gradient \cite{LANLIZHA09} that produce sparse solution, however, these methods cannot be applied to this problem. We introduce a surrogate optimization problem to \eqref{eq:relax-optimization} by adding a loss function. For some large $\lambda\in\real$, solve
\begin{equation}
\begin{aligned}
\widehat{K} = A - \widehat{Z} = \arg\,&\min_{Z,L,R}\; \|Z\|_1 + \lambda\|A-Z-LR^T\|_2^2\\
&\;\text{s.t.}\quad \|L\|_{\infty,2},\|R\|_{\infty,2}\leq 1. 
\end{aligned}
\nonumber
\end{equation}
Here, the matrix $Z$ is sparse and includes the disagreements. For sufficiently large values of $\lambda$, the loss function ensures that the matrix $A-Z$ is close to the matrix $LR^T$ that is a bounded max-norm matrix. To solve this problem iteratively, we use the following update
\begin{equation}
\begin{aligned}
Z_{k+1} &= \P_{\ell_1}\left(Z_k + \frac{\tau\lambda}{\sqrt{k}} (A - Z - LR^T)_k\right)\\
\left[\begin{tabular}{c}$L$\\$R$\end{tabular}\right]_{k+1}\!\!\!\!\!\!\!\! &= \P_{\max}\left(\left[\begin{tabular}{c}$L$\\$R$\end{tabular}\right]_{k}\!\!\!\!\! +\frac{\tau\lambda}{\sqrt{k}}\left[\begin{tabular}{c}$(A-Z-LR^T)\;\;R$\\$(A-Z-LR^T)^T\,L$\end{tabular}\right]_{k}\right).
\end{aligned}
\nonumber
\end{equation}
Here, $\P_{\ell_1}(\cdot)$ operates on entries; if an entry has the same sign before and after the update, it remains unchanged; otherwise, it will be set to zero. Solving directly for large values of $\lambda$ might cause some problems due to the finite numerical precision. In practice, we start with some small value say $\lambda=1$ and double the value of $\lambda$ after some iterations. This way, we gradually put more and more emphasis on the loss function as we get closer to the optimal point.

\subsection{Dual Decomposition Method}

Inspired by \citet{ROC70}, we first reformulate \eqref{eq:relax-optimization} by introducing a dummy variable $Z\in\real^{n\times n}$ as follows
\begin{equation}
\begin{aligned}
\widehat{K} = \arg\,&\min_{Z,K}\; \|A-K\|_1 \\
&\;\text{s.t.}\quad \|Z\|_{\max}\leq 1 \qquad\text{and}\qquad Z=K. 
\end{aligned}
\nonumber
\end{equation}
Then, introducing a Lagrange multiplier $\Lambda\in\real^{n\times n}$, we propose the following equivalent problem:
\begin{equation}
\begin{aligned}
\widehat{K} = \arg\,&\max_{\Lambda}\min_{Z,K}\; \|A-K\|_1 +\tracer{\Lambda}{K-Z} \\
&\;\text{s.t.}\quad \|Z\|_{\max}\leq 1. 
\end{aligned}
\nonumber
\end{equation}
Here, $\tracer{\cdot}{\cdot}$ is the trace of the product. This problem is a saddle-point convex problem in $(Z,K,\Lambda)$. To solve this, we iteratively fix $\Lambda$ and optimize over $(K,Z)$ and then, using those optimal values of $(K,Z)$, update $\Lambda$. 

For a fixed $\Lambda$, the problem can be separated into two optimization problems over $K$ and $Z$ as
\begin{equation}
\widehat{K}(\Lambda) = \arg\,\min_{K}\; \|A-K\|_1 +\tracer{\Lambda}{K}
\nonumber
\end{equation}
which can be solved using factorization method discussed above, and
\begin{equation}
\begin{aligned}
\widehat{Z}(\lambda)=\arg\,&\min_{Z}\; -\tracer{\Lambda}{Z}\\&\;\text{s.t.}\quad \|Z\|_{\max}\leq 1. 
\end{aligned}
\nonumber
\end{equation}
which is a soft thresholding; if $|\Lambda_{ij}|>1$ then, $\widehat{K}({\Lambda})_{ij}=-\sgn(\Lambda_{ij})$; otherwise $\widehat{K}({\Lambda})_{ij}=A_{ij}$. 

Using $\widehat{K}(\Lambda_k)$ and $\widehat{Z}(\Lambda_k)$, we update $\Lambda$ as follows
\begin{equation}
\Lambda_{k+1} = \Lambda_{k} - \frac{\tau}{\sqrt{k}}(\widehat{K}(\Lambda_k) - \widehat{Z}(\Lambda_k))
\nonumber
\end{equation}
until it converges. One criterion for the convergence of this method is to round both matrices $\widehat{K},\widehat{Z}$ and check if they are equal. To use this criterion, we need to initialize the two matrices very differently to avoid the stopping due to the initialization.

\begin{figure}[t]
\centering
\centering
\subfigure[\small$\|\text{Supp}(A-\hat{K})|$]{
\includegraphics[width=0.45\linewidth]{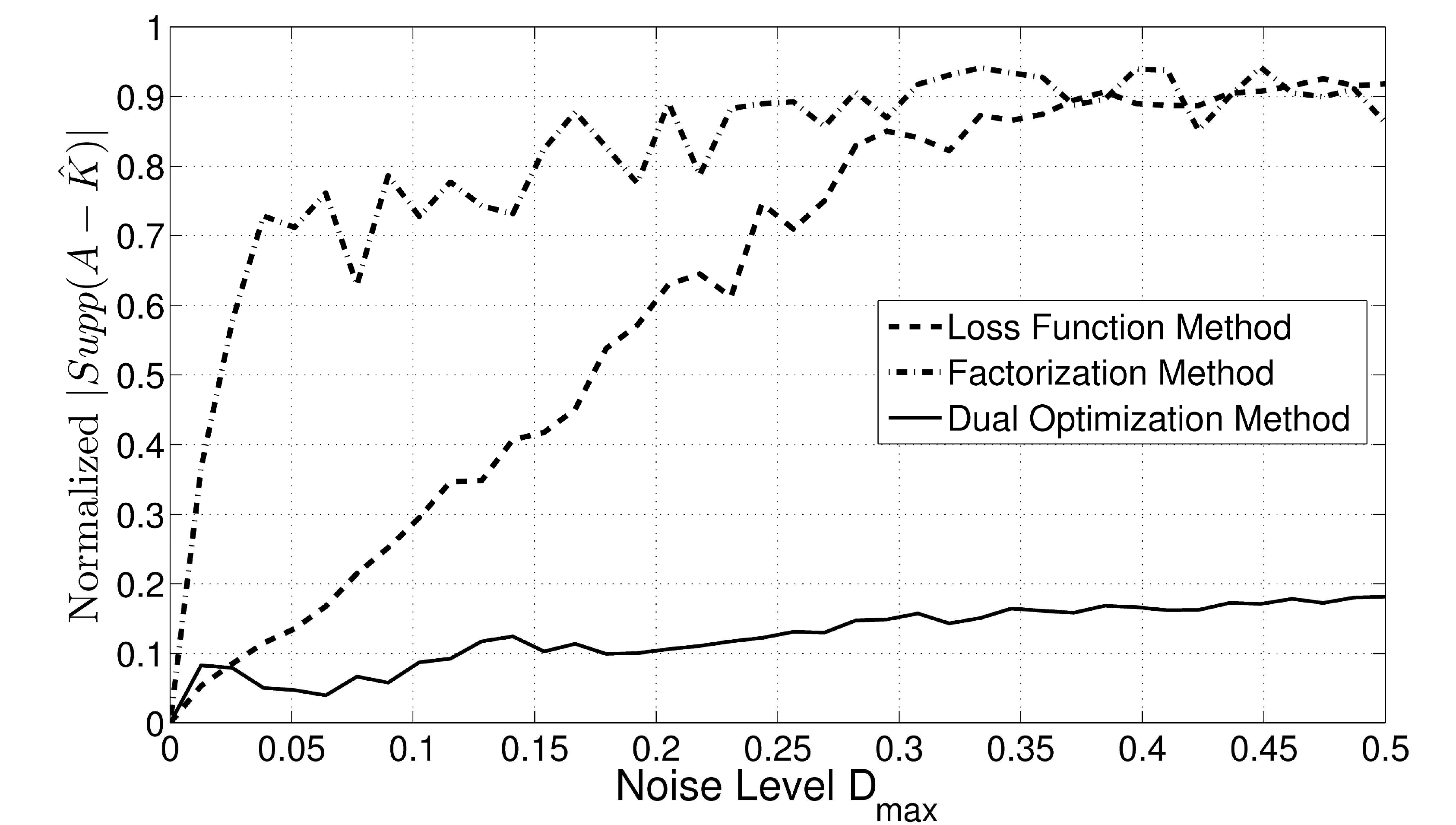}
\label{fig:opt-sparse}
}
\subfigure[\small$\|K^*-\hat{K}\|_1$]{
\includegraphics[width=0.45\linewidth]{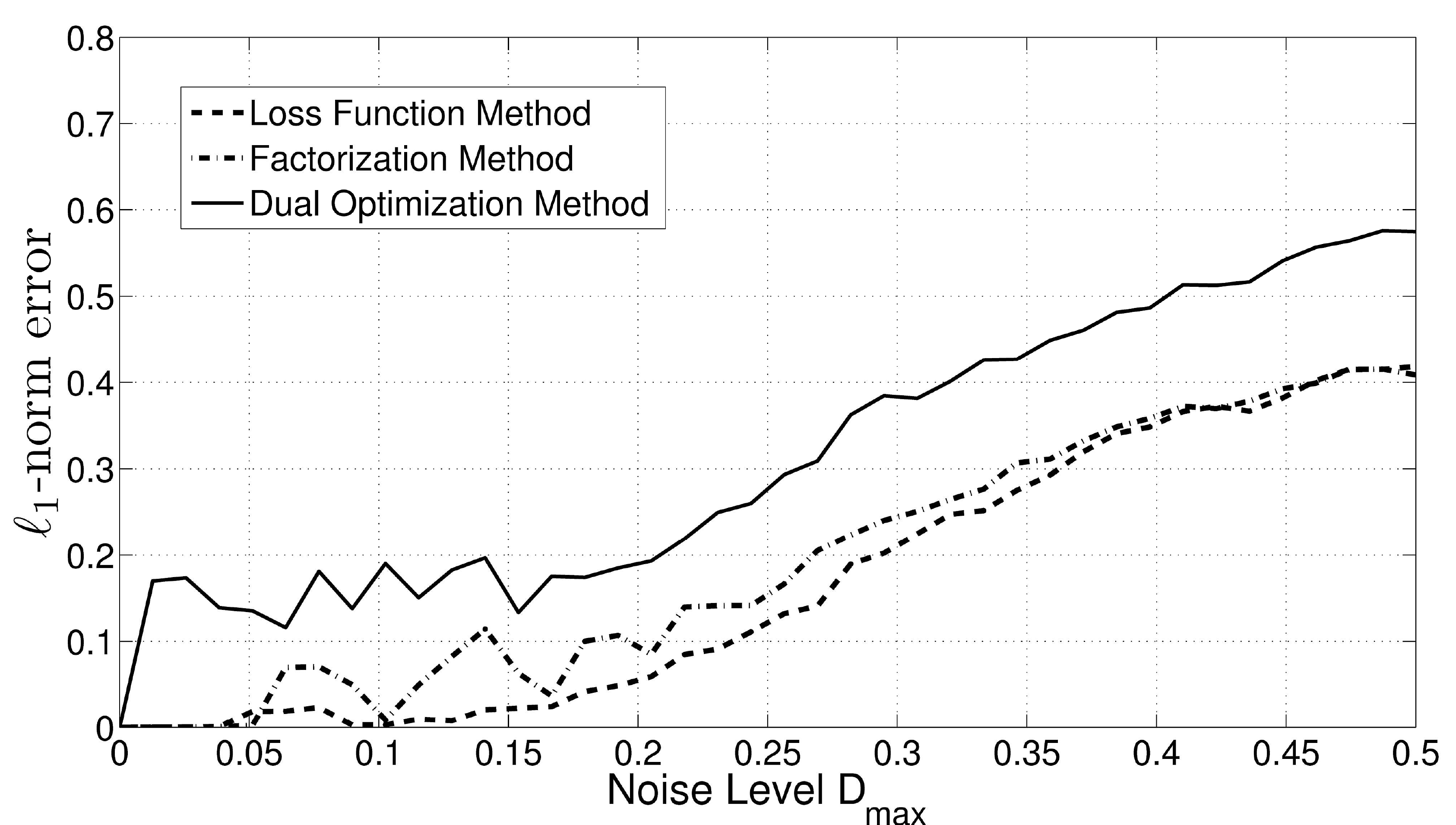}
\label{fig:opt-l1}
}
\caption{Comparison of the proposed numerical optimization methods in terms of the sparsity of the solution they provide and the $\ell_1$ error of the estimation.}
\label{fig:fig1}
\end{figure}

\begin{figure}[t]
\centering
\includegraphics[width=4.0in]{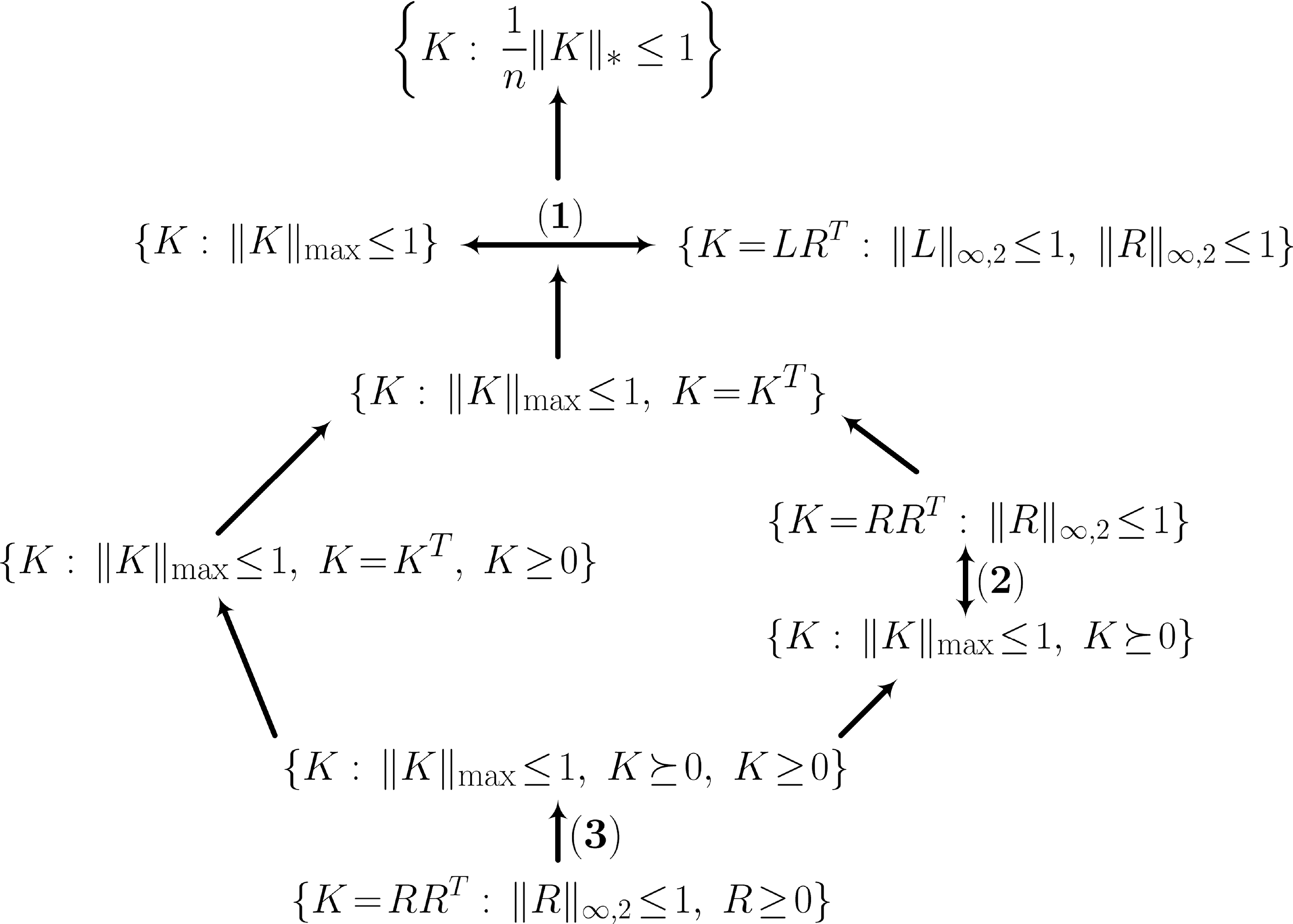}
\caption{Summary of possible convex relaxations of the set of valid clustering matrices and their relations. Here, $\|\cdot\|_*$ represents the trace (nuclear) norm, $\|\cdot\|_{\infty,2}$ represents the maximum $\ell_2$ norm of the rows, ``$\geq$" is used for element-wise positiveness and ``$\succeq$" is used for positive semi-definiteness. Each double-ended arrow represents the equivalence of two sets. Each single-ended arrow in this figure represents a \emph{strict} sub-set relation between two sets.}
\label{fig:conv-relax}
\end{figure}

\subsection{Numerical Comparison}
We compare the performance of these methods.
For three ideal clusters of size $20$ with noise level $D_{\max}$, we
run all three algorithms for $2000$ iterations. We consider an initial
step size $\tau=1$ for all methods, and, for the loss function method, we doubel $\lambda$ every $100$ iterations. For the dual method, we update $\Lambda$ for $20$ times and
run $100$ iterations of the factorization method for the
max-norm sub-problem at each update. We report the sparsity of the solution $A-\widehat{K}$ as well as the $\ell_1$-norm of the error $\|\widehat{K}-K^*\|_1$ for each algorithm in Fig~\ref{fig:fig1}. This result shows that there is a trade-off between sparsity and the error -- the dual optimization method provides consistently a sparse solution, where, factorization and loss function methods provide small error. The sparsity of loss function method gets worse as the noise increases.

\section{Tighter Relaxations}\label{sec:enhancedAlgo}
In this section, we improve our basic algorithm in two ways: first, we use a tighter relaxation for valid clustering constraint and second, we add a single-linkage step after we recovered the clustering matrix. Although max-norm is a tighter relaxation comparing to trace-norm, we would like to go further and introduce tighter relaxations. Figure~\ref{fig:conv-relax} summarizes different possible relaxations based on max-norm. The arrows in this figure indicated the strict subset relations among these relaxations. The tightest relaxation we suggest is $\{K=RR^T: \|R\|_{\infty,2}\leq 1, R\geq 0\}$ based on the intuition that a clustering matrix is symmetric and has a trivial factorization $R\in\real^{n\times k}$, where, $R_{ij}$ is non-zero if node $i$ belongs to cluster $j$. Next lemma formalizes this result. 

\begin{lemma}
All relaxation sets shown in Fig.~\ref{fig:conv-relax} are convex and the strict subset relations hold.
\label{lem:conv-relax}
\end{lemma}

\begin{figure}[t]
\centering
\subfigure[Balanced; Fractional]{
\includegraphics[width=0.45\linewidth]{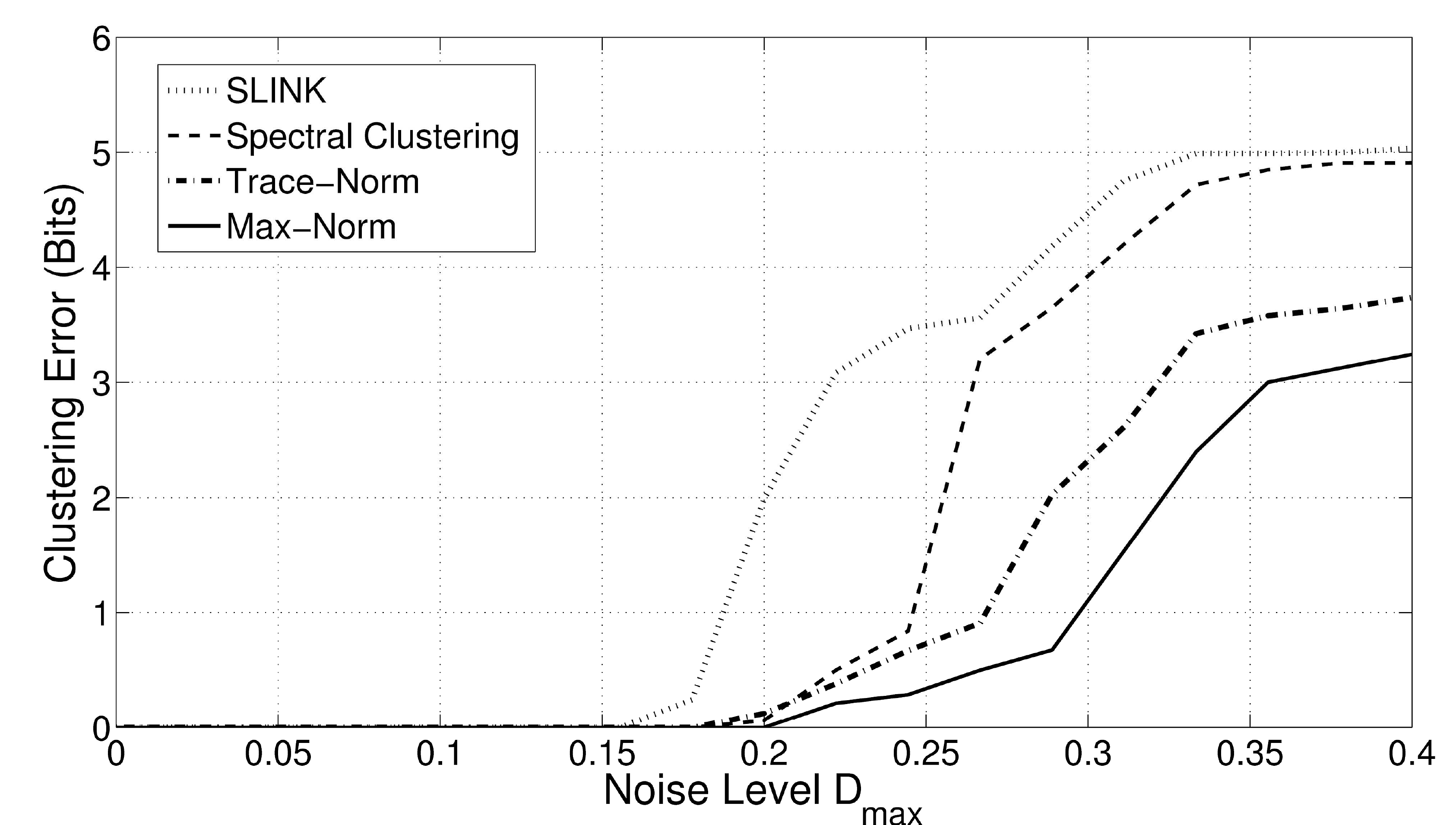}
\label{fig:subfig-bal-frac-entropy}
}
\subfigure[UnBalanced; Fractional]{
\includegraphics[width=0.45\linewidth]{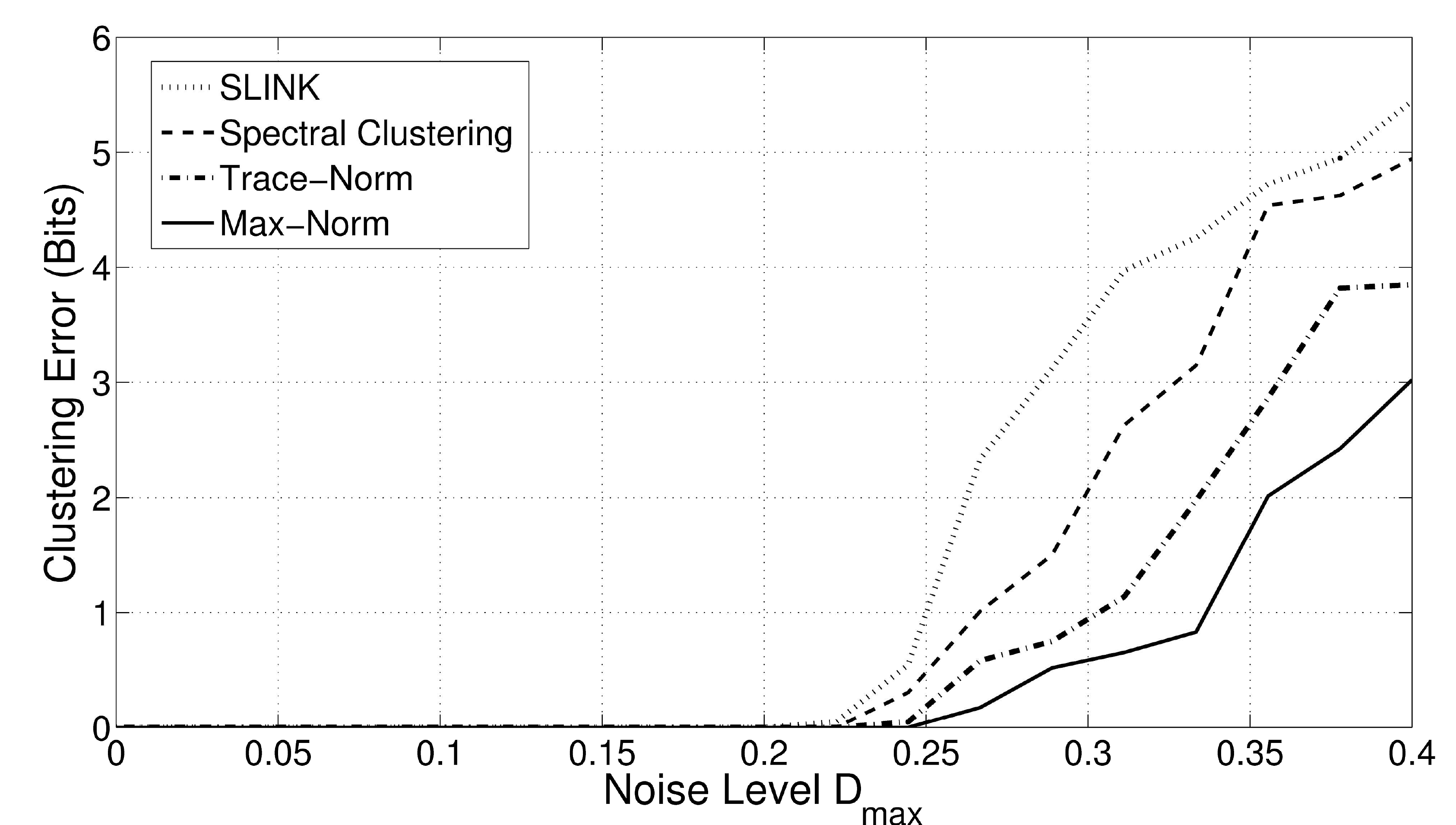}
\label{fig:subfig-unbal-frac-entropy}
}
\label{fig:fig3}
\caption{Comparison of our \emph{best} proposed method which is the linear objective over tight relaxation (followed by a single-linkage algorithm) with trace-norm counterpart, single-linkage algorithm and spectral clustering. Here, we plot the entropy-based distance of the recovered clustering with the underlying true clustering.}
\end{figure}

This suggests using the tightest convex relaxation, that is constraining to $K$ such that there exists $R>=0, \|R\|_{\infty,2}<=1$ with $K=RR^T$ (the set of matrices $K$ with a factorization $K=RR^T, R>=0$ is called the set of \emph{completely positive matrices} and is convex \cite{BERSHA}). We optimize over this relaxation by solving the following optimization problem over $R$:
\begin{equation}
\begin{aligned}
\widehat{R}\; =\;\arg\;&\min_{R}\;\|A-RR^T\|_1\\
&\,\text{s.t.}\quad\|R\|_{\infty,2}\leq 1\quad\&\quad R\geq 0.
\end{aligned}
\label{eq:relax-enhanced}
\end{equation}
and setting $\widehat{K}=\widehat{R}\widehat{R}^T$.  Although the constraint on $\widehat{K}$ {\em is} convex, the optimization problem \eqref{eq:relax-enhanced} is {\em not} convex in $R$.

\subsection{Single-linkage Post Processing}
The matrix $\widetilde{K}$ extracted from \eqref{eq:relax-enhanced} might diverge from a valid clustering matrix in two ways: firstly, it might not have the structure of a valid clustering and secondly, even if it has the structure, the values might not be integer. We run SLINK on $\widetilde{K}$ as a ``rounding scheme" to fix both of the above problems. SLINK gives a sequence of clusterings $\mathcal{C}_1,\ldots,\mathcal{C}_n$. To pick the best clustering, we choose
\begin{equation}
\widehat{K} = \arg\,\min_i \|A-K(\mathcal{C}_i)\|_1.
\label{eq:SLINK-Criterion}
\end{equation}
The matrix $\widetilde{K}$ can be viewed as a refined version of the affinity matrix $A$ and hence the second step of the algorithm can be replaced by other hierarchical clustering algorithms. The criterion of choosing the \emph{best} clustering in the hierarchy comes naturally from the correlation clustering formulation.

\subsection{Comparison with Other Algorithms}
We compare our enhanced algorithm with the trace-norm algorithm \cite{JCSX11} followed by SLINK and SLINK itself. In all cases we pick a clustering from SLINK hierarchy using \eqref{eq:SLINK-Criterion}. The setup is identical to the experiment explained in Section~\ref{sec:trace-comparison}. Fig~\ref{fig:fig2} summarizes the results and shows that our enhanced algorithm outperforms all competitive methods significantly.

Besides the exact recovery of the underlying clustering, we would like to investigate that as noise level $D_{\max}$ increases, how bad the output of our algorithm get. Using ``variation of information" \cite{MEI07} as a distance measure for clusterings, we compare our algorithm with linear objective with trace-norm counterpart, SLINK and spectral clustering\cite{LUX07} for both balanced and unbalanced clusterings described before. For the spectral clustering method, we first find the largest $k=4$ principal components of $A$  and then, run SLINK on principal components. Fig~\ref{fig:fig3} shows the result indicating that max-norm, even when the noise level is high and no method can recover the exact clustering, outputs a clustering that is not far from the true underlying clustering in our metric.

\begin{figure}[t]
\centering
\subfigure[\small Time Complexity]{
\includegraphics[width=0.48\linewidth]{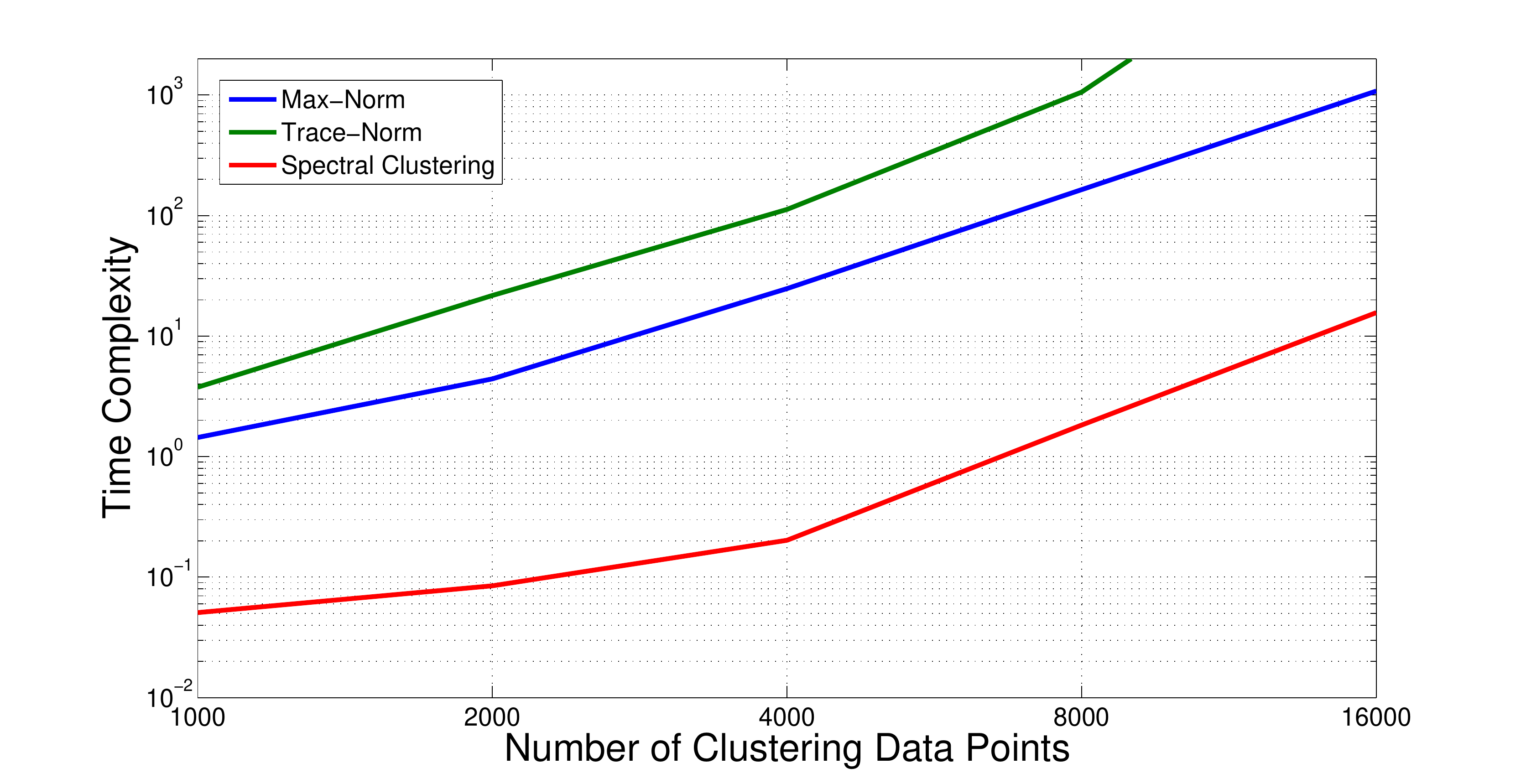}
\label{fig:subfig-MNIST-time}
}
\subfigure[\small Clustering Error]{
\includegraphics[width=0.48\linewidth]{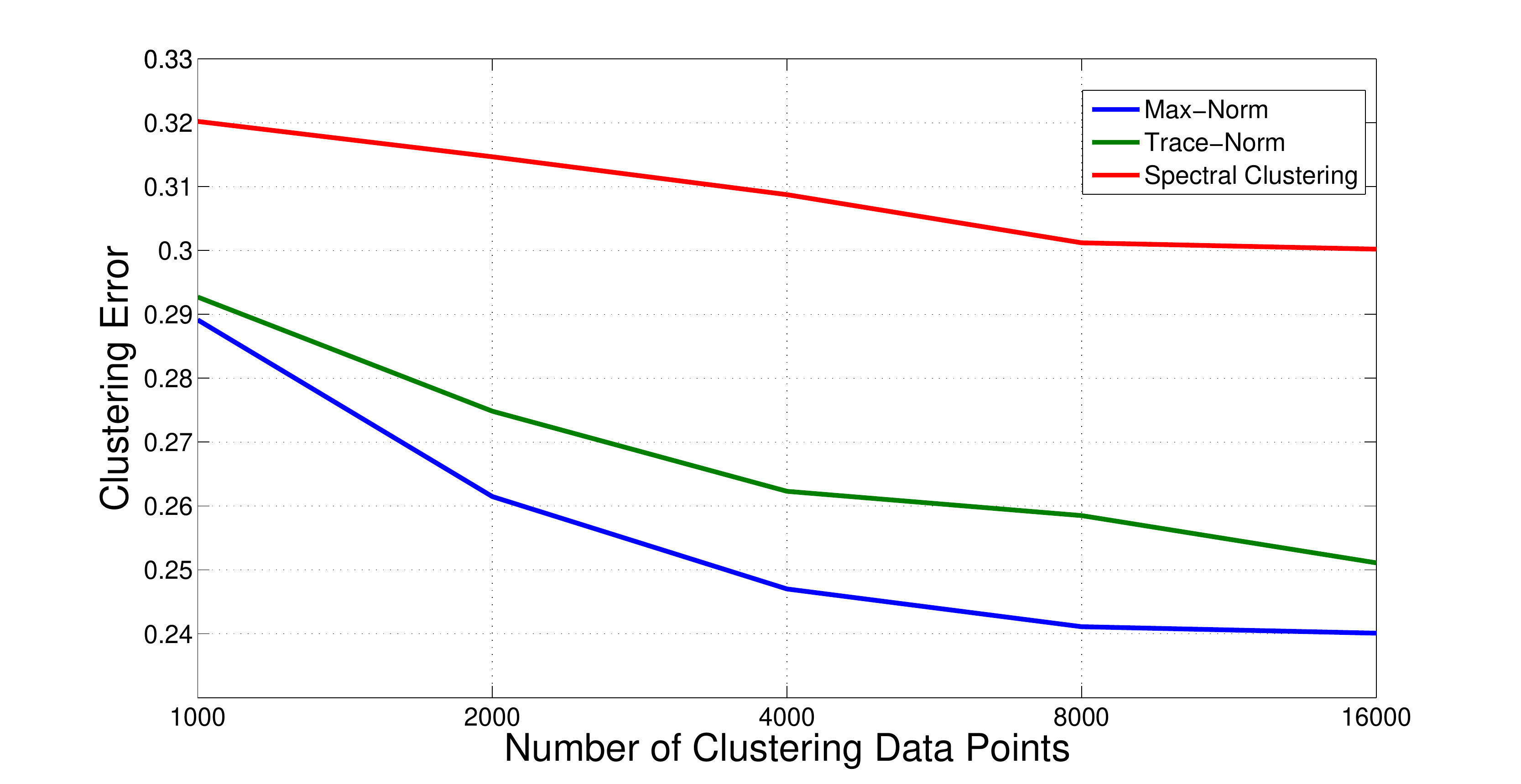}
\label{fig:subfig-MNIST-error}
}
\label{fig:MNIST}
\caption{\small Comparison of our \emph{best} proposed method which is the linear objective over tight relaxation (followed by $k$-means) with trace-norm and spectral clustering in terms of time complexity and clustering error on MNIST dataset.}
\end{figure}

\subsection{MNIST Dataset}

To demonstrate our method in a realistic and larger scale data set, we run our enhanced algorithm, trace-norm and spectral clustering on MNIST Dataset \cite{lecun1998gradient}. For each experiment, we pick a total of $n$ data points from $10$ different classes ($n/10$ from each class) and construct the affinities using Gaussian kernel as explained in \cite{buhler2009spectral}. We report the time complexities and clustering errors as previous experiment in Fig~\ref{fig:MNIST}. For the spectral clustering, we take SVD using Matlab and pick the top 10 principal components followed by $k$-means.

\bibliographystyle{natbib}
\bibliography{MaxNormClusteringarXiv}

\newpage
\appendix

\section{Proof of Lemma~\ref{lem:conv-relax}} \label{sec:prooflemmarelaxation}

Provided equivalences (1) and (2), it is clear that $\{K=LR^T:\|L\|_{\infty,2}\leq 1, \|R\|_{\infty,2}\leq 1\}$ and $\{K=RR^T:\|R\|_{\infty,2}\leq 1\}$ are both convex sets. Since $\{K=RR^T:\|R\|_{\infty,2}\leq 1, R\geq 0\}$ is the intersection of two sets $\{K=RR^T:\|R\|_{\infty,2}\leq 1\}$ and $\mathcal{CP}\{K=RR^T:R\geq 0\}$, it suffices to show that $\mathcal{CP}$ is a convex set. The set $\mathcal{CP}$ is called the set of \emph{completely positive matrices} and has been shown to be a closed convex cone (see Theorem 2.2 in \cite{BERSHA} for details).

For the proof of equivalence (1) see Lemma~15 in \cite{NatiPHD}. To prove equivalence (2), it is clear that $\{K=RR^T:\|R\|_{\infty,2}\leq 1\}\subseteq\{K:\|K\|_{\max}\leq 1, K\succeq 0\}$. Now, suppose $K_0\in\{K:\|K\|_{\max}\leq 1, K\succeq 0\}$; let $R_0=\sqrt{K_0}$ and in contrary, assume that $\|R_0\|_{\infty,2}>1$. This implies that at least one element on the diagonal of $K_0$ exceeds $1$ and hence $\|K_0\|_{\max}>1$. This is a contradiction and hence the equivalence (2) follows.

To show the relation (3), it suffices to show that the sub-set relation is strict, since the sub-set relation itself is trivial. By counter-example provided in \cite{GW80}, the sub-set relation is strict (i.e., there is a positive semi-definite and positive entry $K_0$ that does not belong to $\mathcal{CP}$).

\section{Proof of Lemma~\ref{lem:DmaxBound}} \label{sec:prooflemmadmaximum}
\begin{figure}[t]
\centering
\subfigure[Original Clustering]{
\includegraphics[width=2.5in]{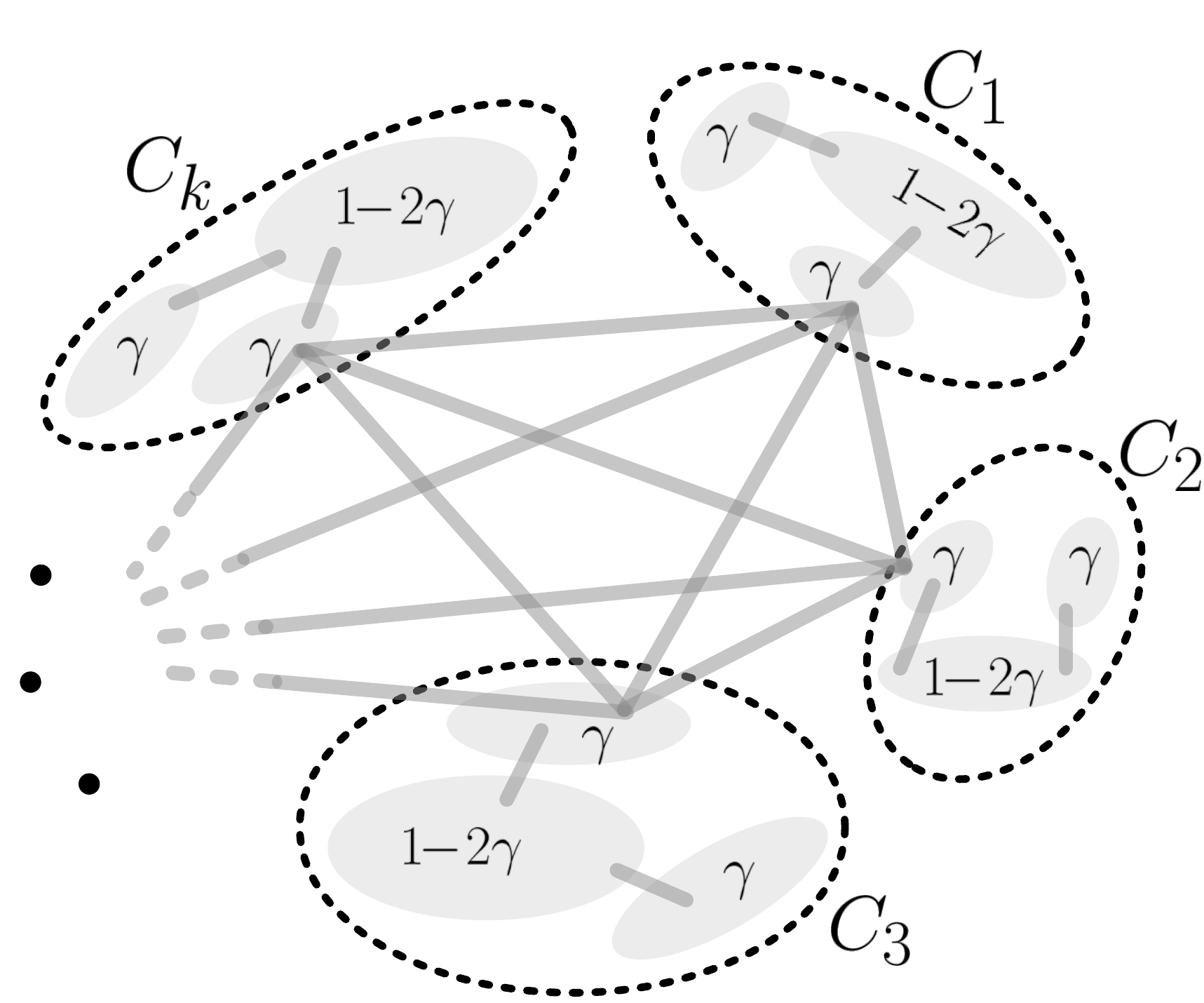}
\label{fig:DmaxLem1}
}
$\qquad\qquad$
\subfigure[Alternative Clustering]{
\includegraphics[width=2.5in]{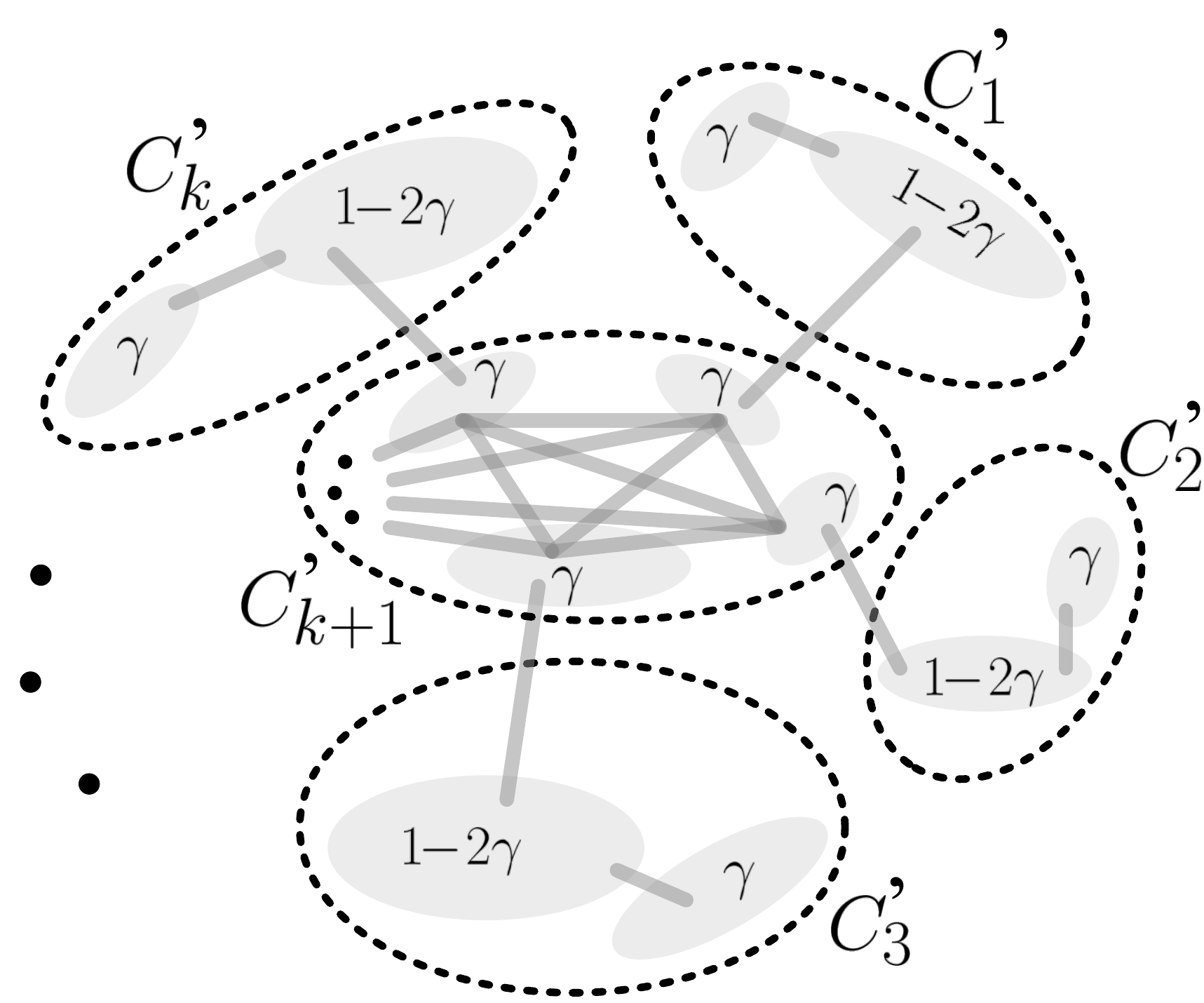}
\label{fig:DmaxLem2}
}
\label{fig:DmaxLem}
\caption{Illustration of two alternative clusterings on the same graph with $D_{\max}=\gamma$. Each gray cloud of points is a clique. Each link between two clouds of points connects every points on one cloud to every points on the other cloud.}
\end{figure}

We construct an example with $D_{\max}=\frac{2}{\frac{n^2}{\sum_i\, |C_i|^2}+5}$ that cannot be recovered. Consider the clustering shown in Fig.~\ref{fig:DmaxLem1}. It is clear that for this clustering, we have $D_{\max}=\gamma$ and
\begin{equation}
B(\mathcal{C}_1) = \gamma^2\sum_{i=1}^k |C_i|^2 + \frac{\gamma^2}{2}\sum_{i=1}^k |C_i|(n-|C_i|).
\nonumber
\end{equation}
Now, consider the alternative clustering shown in Fig.~\ref{fig:DmaxLem2}. For this alternative clustering, we have
\begin{equation}
B(\mathcal{C}_2) = \gamma(1-2\gamma)\sum_{i=1}^k |C_i|^2.
\nonumber
\end{equation}
It is clear that $B(\mathcal{C}_2)<B(\mathcal{C}_1)$ (the alternative is a better clustering) for $\gamma>\frac{2}{\frac{n^2}{\sum_i\, |C_i|^2}+5}$.\\

\section{Proof of Theorem~\ref{thm:deterministic}} \label{sec:prooftheorem1}
The proof has two main steps; in the first step, we characterize a sufficient optimality condition set based on the existence of a dual variable and in the second step, we construct such dual variable. For the sake of the proof, we consider a useful equivalent definition \cite{LS08} of the max norm as
\begin{equation}
\|K\|_{\max} = \max_{X:\|X\|_2\leq 1} \;\|K \circ X\|_{2}
\label{eq:var-max-norm}
\end{equation}
where, $\|\cdot\|_2$ is the spectral norm (maximum eigenvalue) of the matrix and ``$\,\circ\,$" is the Hadamard element-wise product.

\subsection{Notation}
In this section, we introduce our notation and definitions used throughout the paper.

\subsubsection{Residual Matrix Notations}
In general, we do not expect the residual matrix $B^*=A-K^*$ to be sparse unless we threshold the affinity matrix (or we have adjacency matrix). However, to provide a guarantee, we need to characterize the sub-gradient of the $\ell_1$-norm and hence distinguish between zeros and non-zeros of $B^*$. Let 
\begin{equation}
\Omega = \{B\in\real^{n\times n}: B = B^T, \text{\bf Supp}(B)\subseteq \text{\bf Supp}(B^*)\},
\end{equation}
where, $\text{\bf Supp}(\cdot)$ is the index set of non-zero entries. The orthogonal projection of a matrix $M$ to this space is defined to be a matrix of the same size with $\P_\Omega(M)_{ij} = M_{ij}$ if $(i,j)\in\text{\bf Supp}(B^*)$ and zero otherwise. The orthogonal complement of this space is denoted by $\Omega^\perp$ and the projection is defined as $\P_{\Omega^\perp}(M) = M -\P_\Omega(M)$.

\subsubsection{Clustering Matrix Notations}
Let $U\in\real^{n\times k^*}$ be constructed as
\begin{equation}
U=\left[\begin{tabular}{cccc} $\frac{1}{\sqrt{|C_1|}}\mathbf{1}_{|C_1|}$& & & \\ &$\frac{1}{\sqrt{|C_2|}}\mathbf{1}_{|C_2|}$ & & \\ & &$\;\cdot\;$ & \\ & & & $\frac{1}{\sqrt{|C_{k^*}|}}\mathbf{1}_{|C_{k^*}|}$\end{tabular}\right].
\end{equation}
Define $\T = \{UX^T + YU^T: X,Y\in\real^{n\times k^*}\}$ to be the space of matrices sharing either row or column space with $U$. The orthogonal projection to this space can be defined as
\begin{equation}
\P_\T(M) = UU^TM + MUU^T - UU^TMUU^T,
\nonumber
\end{equation}
where,
\begin{equation}
UU^T=\left[\begin{tabular}{cccc} $\frac{1}{|C_1|}\mathbf{1}_{|C_1|\times |C_1|}$& & & \\ &$\frac{1}{|C_2|}\mathbf{1}_{|C_2|\times |C_2|}$ & & \\ & &$\;\cdot\;$ & \\ & & & $\frac{1}{|C_{k^*}|}\mathbf{1}_{|C_{k^*}|\times |C_{k^*}|}$\end{tabular}\right].
\nonumber
\end{equation}
Denote the orthogonal complement of the space $\T$ by $\T^\perp$ equipped with projection $\P_{\T^\perp}(M) = M - \P_\T(M)$. Let $\alpha = 2D_{\max}$ be the contraction between the ideal clusters and disagreements (See Lemma~\ref{lem:space-separation} for more details on this definition). Under the assumption of the theorem, we have $\alpha<1$ and hence, $\T\cap\Omega=\{0\}$.\\

\noindent Using definitions in \eqref{eq:infinite_sums}, let 
\begin{equation}
X^* = W(Z^*) + V(UU^T),
\nonumber
\end{equation}
where,
\begin{equation}
Z^*=\P_\Omega\left[\begin{tabular}{cccc} $\frac{1}{|C_1|}\mathbf{1}_{|C_1|\times |C_1|}$& $\frac{1}{\sqrt{|C_1|\,|C_2|}}\mathbf{1}_{|C_1|\times |C_2|}$ & $\cdot$ & $\frac{1}{\sqrt{|C_1|\,|C_{k^*}|}}\mathbf{1}_{|C_1|\times |C_{k^*}|}$ \\ $\frac{1}{\sqrt{|C_2|\,|C_1|}}\mathbf{1}_{|C_2|\times |C_1|}$ &$\frac{1}{|C_2|}\mathbf{1}_{|C_2|\times |C_2|}$ & $\cdot$ & $\frac{1}{\sqrt{|C_2|\,|C_{k^*}|}}\mathbf{1}_{|C_2|\times |C_{k^*}|}$ \\ $\cdot$ & $\cdot$ &$\;\cdot\;$ & $\cdot$\\ $\frac{1}{\sqrt{|C_{k^*}|\,|C_1|}}\mathbf{1}_{|C_{k^*}|\times |C_1|}$ & $\frac{1}{\sqrt{|C_{k^*}|\,|C_2|}}\mathbf{1}_{|C_{k^*}|\times |C_2|}$ & $\cdot$ & $\frac{1}{|C_{k^*}|}\mathbf{1}_{|C_{k^*}|\times |C_{k^*}|}$\end{tabular}\right].
\nonumber
\end{equation}
Notice that $\P_\T(X^*) = UU^T$ and hence $X^*-UU^T\in\T^\perp$. If we show that $X^*-UU^T$ has spectral norm less than $1$, then it is immediate that $X^* \in \arg\; \max_{X:\|X\|_2\leq 1} \;\|K^* \circ X\|_{2}$. Also, we have an eigenvalue decomposition $K^* \circ X^* = [U\; V]\Sigma [U\; V]^T$, where, $U$ is as defined above and contains the eigenvector(s) corresponding to the maximum magnitude eigenvalue $+1$ (with $k^*$ repetitions). To bound the spectral norm of $X^*-UU^T$, consider
\begin{equation}
\begin{aligned}
\|X^*-UU^T\|_2&=\|W(\P_\Omega(Z^*-UU^T))\|_2\\
&\leq \frac{D_{\max}}{1-\alpha}(k^*-1)<1.
\end{aligned}
\nonumber
\end{equation} 
The first inequality follows from Lemma~\ref{lem:spec-norm-bound}. We make assumptions so that the last inequality holds.\\

\noindent We use the variational form \eqref{eq:var-max-norm} to characterize the sub-gradient of the max-norm at the point $K^*$. 

\begin{lemma}
For a matrix $M\in\real^{n\times n}$, we have $M\in\partial\|K^*\|_{\max}$ if $M = (USU^T + W) \circ X^*$, for some diagonal positive semi-definite matrix $S\in\real^{r\times r}$ with $\trace{S}=1$ and for some matrix $W\in\real^{n\times n}$ with $\P_\T(W)=0$ and $\|W\|_*<1$.
\label{lem:max-norm-subg}
\end{lemma}

\begin{proof}
Using the variational form \eqref{eq:var-max-norm} and theorem 4.4.2 in \cite{HIRLEM} on the sub-gradient of the maximum of convex functions, we have
\begin{equation}
\partial\|K^*\circ X^*\|_2\, \subseteq\, \partial\|K^*\|_{\max}.
\nonumber
\end{equation}
Thus, it suffices to show that $M\in\partial\|K^*\circ X^*\|_2$ (which is the case).\\
\end{proof}

\subsection{Sufficient Optimality Conditions}
We provide similar optimality conditions to those provided in $\ell_1$ plus trace norm minimization in the literature. The main difference here is the existence of the auxiliary variable $X^*$ in the conditions. The following lemma characterizes a sufficient optimality condition set.

\begin{lemma} [Sufficient Optimality Condition.]
$K^*=\widehat{K}_\mu$ (Problem \eqref{eq:orig-optimization}$\,\equiv\,$ Problem \eqref{eq:equiv-relaxed-optimization}), if $\T\cap\Omega = \{0\}$ and there exists a dual matrix $Q$ such that
\begin{itemize}
\item [(a)] $\P_\Omega(Q\circ X^*) = -\frac{1-\mu}{n^2}\sgn(A-K^*)$
\item [(b)] $\left\|\P_{\Omega^\perp}(Q\circ X^*)\right\|_\infty < \frac{1-\mu}{n^2}$
\item [(c)] $\P_\T(Q) = USU^T$, for some diagonal matrix $S\succeq 0$ with $\text{Trace}(S)=\mu$. 
\item [(d)] $\left\|\P_{\T^\perp}(Q)\right\|_*<\mu$.
\end{itemize}
\label{lem:suff-opt}
\end{lemma}

\begin{proof}
Notice that since $X^*$ by construction has no zero entry (except for the very corner case where there are only two clusters both of size $2$), the matrix $Q\circ X^*$ can take any value/sign on each entry by choosing the values of $Q$ properly. Under these conditions, $Q\circ X^*\,\in\partial\|A-K^*\|_1$ and also $Q\circ X^*\,\in\partial\|K^*\|_{\max}$ and the result follows from the standard first order optimality argument and zero duality gap of both $\ell_1$ and max norms.\\
\end{proof}

\subsection{Dual Variable Construction}
First notice that under the assumption of the theorem, we have $\alpha<1$ and hence, by Lemma~\ref{lem:space-separation}, we have $\T\cap\Omega=\{0\}$ and also $\mu=\mu_0$ is feasible. Second, we construct $Q$ by using alternating projections. Consider the infinite sums
\begin{equation}
\begin{aligned}
W(M) &= M - \P_\T(M) + \P_\Omega(\P_\T(M)) - \P_\T(\P_\Omega(\P_\T(M))) +\ldots\\
V(N) &= N\, - \P_\Omega(N)\, + \P_\T(\P_\Omega(N))\, - \P_\Omega(\P_\T(\P_\Omega(N)))\, + \ldots
\end{aligned}
\label{eq:infinite_sums}
\end{equation}
By the proof of the  Lemma~\ref{lem:space-separation}, these sums converge geometrically with parameter $\alpha$ (See Lemma 5 in \cite{JCSX11} for the proof). Denoting element-wise division with ``$/$" (and $\frac{0}{0}=0$), let
\begin{equation}
Q = -\frac{1-\mu}{n^2}W(\sgn(A-K^*)/X^*) + \frac{\mu}{k^*}V(UU^T).
\nonumber
\end{equation}
It is easy to check that conditions (a) and (c) in lemma~\ref{lem:suff-opt} are both satisfied for $S=\frac{1}{k^*}\mathbf{I}$. To show condition (b), first notice that $\|\P_{\Omega^\perp}\P_{\T}\P_{\Omega}(M)\|_\infty\leq D_{\max}\|\P_{\Omega^\perp}(M)\|_\infty$ and hence, we have
\begin{equation}
\begin{aligned}
\|\P_{\Omega^\perp}(Q\circ X^*)\|_\infty&\leq\frac{1}{(1-D_{\max})^2} \left\|\P_{\Omega^\perp}\left(\frac{1-\mu}{n^2}\P_\T(\sgn(A-K^*)/X^*) + \frac{\mu}{k^*}UU^T\right)\circ \P_{\Omega^\perp}\left(UU^T-\P_\T(Z^*)\right)\right\|_\infty\\
& = \max_i\frac{1}{(1-D_{\max})^2}\left(\frac{(1-\mu)|C_i|}{n^2}D_{\max} +\frac{\mu}{k^*}\frac{1}{|C_i|}\right)\frac{1+D_{\max}}{|C_{i}|}\\
& = \frac{1}{(1-D_{\max})^2}\left(\frac{(1-\mu)}{n^2}(1+D_{\max})D_{\max} +\frac{\mu}{k^*}\frac{1+D_{\max}}{|C_{\min}|^2}\right) < \frac{1-\mu}{n^2}.
\end{aligned}
\nonumber
\end{equation}
The last inequality holds for $\frac{(1-\mu)k^*}{\mu n^2} >\frac{(1+D_{\max})}{(1-3D_{\max})|C_{\min}|^2}$. For the condition (d), we have

\footnotesize\begin{equation}
\begin{aligned}
&\|\P_{\T^\perp}(Q)\|_*\leq\frac{1}{1-\alpha}\left\|\P_{\T^\perp}\left(\frac{1-\mu}{n^2}\sgn(A-K^*)/X^* + \frac{\mu}{k^*}\P_\Omega(UU^T)\right)\right\|_*\\ &\leq\frac{1}{1-\alpha}\left\|\P_\Omega\left[ \begin{tabular}{cccc}$\left(\frac{\mu}{k^*}\frac{1}{|C_1|}-\frac{1-\mu}{n}\frac{|C_1|}{n}\right)\mathbf{1}_{|C_1|\times|C_1|}$& $-\frac{1-\mu}{n}\frac{\sqrt{|C_1|\,|C_2|}}{n}\mathbf{1}_{|C_1|\times|C_2|}$ &$\cdot$& $-\frac{1-\mu}{n}\frac{\sqrt{|C_1|\,|C_{k^*}|}}{n}\mathbf{1}_{|C_1|\times|C_{k^*}|}$ \\$-\frac{1-\mu}{n}\frac{\sqrt{|C_2|\,|C_1|}}{n}\mathbf{1}_{|C_2|\times|C_1|}$ &$\left(\frac{\mu}{k^*}\frac{1}{|C_2|}-\frac{1-\mu}{n}\frac{|C_2|}{n}\right)\mathbf{1}_{|C_2|\times|C_2|}$& $\cdot$& $-\frac{1-\mu}{n}\frac{\sqrt{|C_2|\,|C_{k^*}|}}{n}\mathbf{1}_{|C_2|\times|C_{k^*}|}$\\ $\cdot$&$\cdot$&$\cdot$&$\cdot$\\ $-\frac{1-\mu}{n}\frac{\sqrt{|C_{k^*}|\,|C_1|}}{n}\mathbf{1}_{|C_{k^*}|\times|C_1|}$ &$-\frac{1-\mu}{n}\frac{\sqrt{|C_{k^*}|\,|C_2|}}{n}\mathbf{1}_{|C_{k^*}|\times|C_2|}$& $\cdot$& $\left(\frac{\mu}{k^*}\frac{1}{|C_{k^*}|}-\frac{1-\mu}{n}\frac{|C_{k^*}|}{n}\right)\mathbf{1}_{|C_{k^*}|\times|C_{k^*}|}$\end{tabular} \right]\right\|_*\\ &\leq\frac{D_{\max}}{1-\alpha}\frac{1-\mu}{n}\left\|\left[ \begin{tabular}{cccc}$\frac{|C_1|}{n}\mathbf{1}_{|C_1|\times|C_1|}$& $\frac{\sqrt{|C_1|\,|C_2|}}{n}\mathbf{1}_{|C_1|\times|C_2|}$ &$\cdot$& $\frac{\sqrt{|C_1|\,|C_{k^*}|}}{n}\mathbf{1}_{|C_1|\times|C_{k^*}|}$ \\$\frac{\sqrt{|C_2|\,|C_1|}}{n}\mathbf{1}_{|C_2|\times|C_1|}$ &$\frac{|C_2|}{n}\mathbf{1}_{|C_2|\times|C_2|}$& $\cdot$& $\frac{\sqrt{|C_2|\,|C_{k^*}|}}{n}\mathbf{1}_{|C_2|\times|C_{k^*}|}$\\ $\cdot$&$\cdot$&$\cdot$&$\cdot$\\ $\frac{\sqrt{|C_{k^*}|\,|C_1|}}{n}\mathbf{1}_{|C_{k^*}|\times|C_1|}$ &$\frac{\sqrt{|C_{k^*}|\,|C_2|}}{n}\mathbf{1}_{|C_{k^*}|\times|C_2|}$& $\cdot$& $\frac{|C_{k^*}|}{n}\mathbf{1}_{|C_{k^*}|\times|C_{k^*}|}$\end{tabular} \right]\right\|_*\\ &\qquad\qquad\qquad\qquad+\frac{D_{\max}}{1-\alpha}\left\|\left[ \begin{tabular}{cccc}$\frac{\mu}{k^*}\frac{1}{|C_1|}\mathbf{1}_{|C_1|\times|C_1|}$& $\mathbf{0}$ &$\cdot$& $\mathbf{0}$ \\$\mathbf{0}$ &$\frac{\mu}{k^*}\frac{1}{|C_2|}\mathbf{1}_{|C_2|\times|C_2|}$& $\cdot$& $\mathbf{0}$\\ $\cdot$&$\cdot$&$\cdot$&$\cdot$\\ $\mathbf{0}$ &$\mathbf{0}$& $\cdot$& $\frac{\mu}{k^*}\frac{1}{|C_{k^*}|}\mathbf{1}_{|C_{k^*}|\times|C_{k^*}|}$\end{tabular} \right]\right\|_*\\ &= \frac{D_{\max}}{1-\alpha}\left(\frac{1-\mu}{n}\frac{\sum_i|C_i|^2}{n} + \mu\right) < \mu.
\end{aligned}
\nonumber
\end{equation}\normalsize
The last inequality holds for $\frac{(1-\mu)k^*}{\mu n^2}<\frac{(1-\alpha-D_{\max})k^*}{D_{\max}\sum_i|C_i|^2}$ as assumed.\\

\begin{lemma}
If $\alpha<1$ then $\T\cap\Omega=\{0\}$.
\label{lem:space-separation}
\end{lemma}

\begin{proof}
We show that the projection $\P_\T\P_\Omega(\cdot)$ has a norm $\alpha$ strictly less than one. Then, if there exists a non-zero matrix $M\in\T\cap\Omega$, then $\|M\|_\infty=\|\P_\T\P_\Omega(M)\|_\infty\leq\alpha\|M\|_\infty<\|M\|_\infty$ is a trivial contradiction. Let $M\in\Omega$ and consider
\footnotesize\begin{equation}
\begin{aligned}
\left\|\P_\T(M)\right\|_\infty &=  \max_{i,j} \left\|\frac{1}{|C_i|}\mathbf{1}_{|C_i|\times|C_i|}M_{C_i,C_j} + \frac{1}{|C_j|}M_{C_i,C_j}\mathbf{1}_{|C_j|\times|C_j|} - \frac{1}{|C_i|\,|C_j|}\mathbf{1}_{|C_i|\times|C_i|}M_{C_i,C_j}\mathbf{1}_{|C_j|\times|C_j|}\right\|_\infty\\
&\leq 2D_{\max}\|M\|_\infty = \alpha\|M\|_\infty.
\end{aligned}
\nonumber
\end{equation}\normalsize
The last step is attained by optimizing over $|C_i|$ and $|C_j|$. This concludes the proof of the lemma.\\
\end{proof}

\begin{lemma}
$\|W(\P_\Omega(Z^*-UU^T))\|_2 \leq \frac{D_{\max}}{1-\alpha}(k^*-1)$.
\label{lem:spec-norm-bound}
\end{lemma}

\begin{proof}
For $M\in\Omega$, we have $\|M\|_2\leq\|M_\sigma\|_2$, where, $M_\sigma\in\real^{k^*\times k^*}$ with $(M_\sigma)_{i,j} = \|M_{C_i,C_j}\|_2$. By definition of $D_{\max}$, we have $\|M_{C_i,C_j}\|_2\leq D_{\max}\sqrt{|C_i|\,|C_j|}\|M_{C_i,C_j}\|_\infty$. Thus, 
\begin{equation}
\begin{aligned}
\|\P_\Omega(Z^*-UU^T)\|_2&\leq D_{\max}\left\|\left[\begin{tabular}{cccc} $0$& $1$ & $\cdot$ & $1$ \\ $1$ & $0$ & $\cdot$ & $1$ \\ $\cdot$&$\cdot$ &$\cdot$ &$\cdot$ \\ $1$& $1$& $\cdot$ & $0$\end{tabular}\right]\right\|_2\\
&=D_{\max}(k^*-1).
\end{aligned}
\nonumber
\end{equation}
The rest of the proof is straight forward as follows
\begin{equation}
\begin{aligned}
\|W(\P_\Omega(Z^*-UU^T))\|_2 &= \left\|\P_{\T^\perp}\P_\Omega\left(\sum_{i=0}^\infty(\P_\T\P_\Omega)^i(\P_\Omega(Z^*-UU^T))\right)\right\|_2\\
&\leq\left\|\P_\Omega\left(\sum_{i=0}^\infty(\P_\T\P_\Omega)^i(\P_\Omega(Z^*-UU^T))\right)\right\|_2\\
&=\frac{D_{\max}}{1-\alpha}(k^*-1).
\end{aligned}
\nonumber
\end{equation}
This concludes the proof of the lemma.\\
\end{proof}

\end{document}